\documentclass{article}

\PassOptionsToPackage{numbers,compress}{natbib}

\usepackage[final]{neurips_2021}

\usepackage[utf8]{inputenc} 
\usepackage[T1]{fontenc}    
\usepackage{hyperref}       
\usepackage{url}            
\usepackage{booktabs}       
\usepackage{amsfonts}       
\usepackage{nicefrac}       
\usepackage{microtype}      
\usepackage[table,svgnames]{xcolor} 
\usepackage{amsthm,algpseudocode}

\usepackage{graphicx}
\usepackage{caption}
\usepackage{subcaption}
\usepackage{multirow}
\usepackage{wrapfig}

\usepackage{tikz}

\usetikzlibrary{bayesnet}
\usetikzlibrary{arrows}
\usetikzlibrary{positioning}
\newtheorem{theorem}{Theorem}[section]
\newtheorem*{theorem*}{Theorem}
\newtheorem*{remark}{Remark}

\newtheorem{lemma}[theorem]{Lemma}
\newtheorem{definition}{Definition}[section]
\newcommand\indep{ \perp \!\!\! \perp}

\usepackage[ruled,vlined]{algorithm2e}
\usepackage{mathtools}

\usetikzlibrary{arrows,automata} 
\usepackage[export]{adjustbox}

\definecolor{mpigreen}{HTML}{007977}
\definecolor{mBeige}{HTML}{F5F5DC}
\definecolor{mHoneydew}{HTML}{F0FFF0}
\definecolor{mFlowerwhite}{HTML}{FFFAF0}
\definecolor{mBlue}{HTML}{1F77B4}
\definecolor{mYellow}{HTML}{FF7F0E}

\usepackage{array} 
\usepackage{booktabs} 

\usepackage[column=0]{cellspace} 
\makeatletter
\renewcommand{\Function}[2]{%
  \csname ALG@cmd@\ALG@L @Function\endcsname{#1}{#2}%
  \def\jayden@currentfunction{#1}%
}
\newcommand{\funclabel}[1]{%
  \@bsphack
  \protected@write\@auxout{}{%
    \string\newlabel{#1}{{\jayden@currentfunction}{\thepage}}%
  }%
  \@esphack
}
\makeatother

\usepackage{amsmath,amsfonts,bm}

\def\eqref#1{equation~\ref{#1}}

\def\1{\bm{1}}

\DeclareMathAlphabet{\mathsfit}{\encodingdefault}{\sfdefault}{m}{sl}
\SetMathAlphabet{\mathsfit}{bold}{\encodingdefault}{\sfdefault}{bx}{n}

\ifdefined\eqdef
\else
    \newcommand{\eqdef}
    {
        \overset{{\rm \mbox{\tiny def}}}{=}
    }
\fi

\newcommand{\ie}{{\em i.e.\/}}
\newcommand{\eg}{{\em e.g.\/}}
\newcommand{\xb}{{\mathbf x}}

\newcommand{\ib}{{\mathbf{i}}}
\newcommand{\ibc}{{\overline{\mathbf{i}}}}
\newcommand{\jbc}{{\overline{\mathbf{j}}}}
\newcommand{\jb}{{\mathbf{j}}}

\newcommand{\Acal}{{\mathcal A}}

\newcommand{\Ncal}{{\mathcal N}}

\newcommand{\Wcal}{{\mathcal W}}

\newcommand{\Zcal}{{\mathcal Z}}

\newcommand{\EE}{{\mathbb{E}}}

\newcommand{\Reals}{{\mathbb{R}}}

\title{Generalization Bounds For Meta-Learning:\\ An Information-Theoretic Analysis}

\author{%
    Qi Chen\thanks{\texttt{qi.chen.1@ulaval.ca}}\\
  Université Laval\\
   \And
   Changjian Shui \\
  Université Laval\\
   \And
   Mario Marchand \\
  Université Laval\\
}

\begin{document}

\maketitle

\begin{abstract}

We derive a novel information-theoretic analysis of the generalization property of meta-learning algorithms. Concretely, our analysis proposes a generic understanding of both the conventional learning-to-learn framework \citep{amit2018meta} and the modern model-agnostic meta learning (MAML) algorithms \citep{finn2017model}.
Moreover, we provide a data-dependent generalization bound for a stochastic variant of MAML, which is \emph{non-vacuous} for deep few-shot learning. As compared to previous bounds that depend on the square norm of gradients, empirical validations on both simulated data and a well-known few-shot benchmark show that the proposed bound is orders of magnitude tighter in most situations.
\end{abstract}

\section{Introduction}

Learning a task with limited samples is crucial for real-world machine learning applications, where proper \emph{prior knowledge} is a key component for a successful transfer. Meta-Learning \citep{thrun1998learning} or learning-to-learn (LTL) aims to extract such information through previous training tasks, which has recently re-emerged as an important topic.

Modern approaches based on MAML~\citep{finn2017model} have gained tremendous success by exploiting the capabilities of deep neural networks \citep{liu2018darts,brown2020language,garcia2017few,ravi2018amortized,snell2017prototypical,sung2018learning}. However, many theoretical questions still remain elusive. For instance, in the most popular methods for \emph{few-shot learning} \citep{ravi2016optimization}, the task-specific parameters and meta-parameter are updated in support (also called meta-train) and query (also called meta-validation) set, respectively. However, the majority of existing theoretical results such as \citep{amit2018meta,baxter2000model,pentina2014pac,denevi2019learning,balcan2019provable} do not provide a formal understanding of such popular practice. Moreover, modern meta-learning approaches have incorporated  over-parameterized deep neural networks, where conducting the theoretical analysis becomes even more challenging.


In this paper, we introduce a novel theoretical understanding of the generalization property of meta-learning through an information-theoretical perspective \citep{xu2017information}. Compared with previous theoretical results, the highlights of our contributions are as follows:

\textbf{Unified Approach} We analyze two popular scenarios. 1) The conventional LTL \citep{baxter2000model}, where the meta-parameters and task-specific parameters are updated within the same data set (referred as \emph{joint training}). 2) The modern MAML-based approaches where the meta-parameters and task specific parameters are updated on distinct data sets (referred as \emph{alternate training}), and for which the existing theoretical analysis is rare.



\textbf{Flexible Bounds} The proposed meta-generalization error bounds are highly flexible: they are algorithm-dependant, data-dependant, and are valid for non-convex loss functions. \textbf{1)} Specifically, the generalization error bound for joint-training (Theorem \ref{thm:mi-jt}) is controlled by the mutual information between the \emph{output of the randomized algorithm} and \emph{the whole data set}. It can cover the typical results of \citep{pentina2014pac,amit2018meta}, which can be interpreted with an environment-level and a task-level error. In addition, it reveals the benefit of meta learning compared to single task learning.  \textbf{2)} Moreover, the generalization error bound for alternate-training (Theorem \ref{thm:cmi-al}) is characterized by the conditional mutual information between the \emph{output of the randomized algorithm} and the \emph{meta-validation dataset}, conditioned on the \emph{meta-train dataset}. Intuitively, when the outputs of a meta learning algorithm w.r.t. different input data-sets are similar (\ie~the algorithm is stable w.r.t. the data), the meta-generalization error bound will be small. This theoretical result is coherent with the recently-proposed Chaser loss in Bayes MAML \citep{yoon2018bayesian}. 




\textbf{Non-vacuous bounds for gradient-based few-shot learning}~Conventional gradient-based meta-learning theories heavily rely on the assumption of a Lipschitz loss. However, \citep{scaman2018lipschitz} pointed out that this Lipschitz constant for simple neural networks can be extremely large. Thus, conventional gradient-based upper bounds are often \emph{vacuous} for deep few-shot scenarios.  
In contrast, we propose a tighter data-depend bound that depends on the expected \emph{gradient-incoherence} rather than the gradient norm (the approximation of the Lipschitz constant)~\citep{li2019generalization} for the Meta-SGLD algorithm, which is a stochastic variant of MAML that uses the Stochastic Gradient Langevin Dynamics (SGLD) \cite{welling2011bayesian}. We finally validate our theory in few-shot learning scenarios and obtain orders of magnitude tighter bounds in most situations, compared to conventional gradient-based bounds.

\section{Related Work}
\textbf{Conventional LTL} The early theoretic framework, introduced by \citet{baxter2000model}, proposed the notion of \emph{task environment} and derived uniform convergence bounds based on the capacity and covering numbers of function classes. \citet{pentina2014pac} proposed PAC-Bayes risk bounds that depend on  environment-level and task-level errors. \citet{amit2018meta} extended this approach and provided a tighter risk bound. However, their theory applies to stochastic neural networks and used factorized Gaussians to approximate the parameters' distributions, which is computationally expensive to use in practice. \citet{jose2020information} first analyzed meta-learning through information-theoretic tools, while they applied the assumptions that hide some probabilistic relations and obtained theoretical results substantially different from those presented here. Limited to the space, a more detailed discussion is provided in Appendix \ref{jose_related}.

\textbf{Gradient based meta-learning} In recent years, gradient-based meta-learning such as MAML \citep{finn2017model} have drawn increasing attention since they are model-agnostic and are easily deployed for complex tasks like reinforcement learning, computer vision, and federate learning \citep{jiang2019improving,liu2019taming,fallah2020personalized,gui2018few}. Then, Reptile \citep{nichol2018reptile} provided a general first-order gradient calculation method. Other methods combine MAML and Bayesian methods through structured variational inference \citep{finn2018probabilistic} and empirical Bayes \citep{grant2018recasting}. In Bayes MAML\citep{yoon2018bayesian}, they propose a fast Bayesian adaption method using Stein variational gradient descent and conceived a Chaser loss which coincides with the proposed Theorem \ref{thm:cmi-al}.

On the theoretical side, \citet{denevi2019learning} analyzed the average excess risk for Stochastic Gradient Descent  (SGD) with Convex and Lipschitz loss. \citet{balcan2019provable} studied meta-learning through the lens of online convex optimization, and has provided a guarantee with a regret bound. \citet{khodak2019adaptive} extended to more general settings where the task-environment changes dynamically or the tasks share a certain geometric structure. Other guarantees for online meta-learning scenarios are provided by \citet{denevi2019online} and \citet{finn2019online}. Finally, \citep{fallah2020convergence,ji2020multi} also provided a convergence analysis for MAML-based methods.

\textbf{On meta train-validation split} Although the support query approaches are rather difficult to analyze,  some interesting works have appeared on the simplified linear models. \citet{denevi2018learning} first studied train-validation split for linear centroid meta-learning. They proved a generalization bound and concluded that there exists a trade-off for train-validation split, which is consistent with Theorem \ref{thm:cmi-al} in our paper. \citet{bai2021important} applied the random matrix theoretical analysis for a disentangled comparison between joint training and alternate training under the realizable assumption in linear centroid meta-learning. By calculating the closed-form concentration rates over the mean square error of parameter estimation for the two settings, they obtained a better rate constant with joint training. However, we aim to provide a generic analysis and do not make such a realizable assumption. We believe an additional excess risk analysis with more assumptions is needed for a similar comparison, which is out of the scope of this article. Moreover, \citet{saunshi2021representation} analyzed the train-validation split for linear representation learning. They showed that the train-validation split encourages learning a low-rank representation. More detailed discussion and comparison can be found in Appendix \ref{sq-related}.

\textbf{Information-theoretic learning for single tasks} We use here an information-theoretic approach, introduced by \citet{russo2019much} and \citet{xu2017information}, for characterizing single-task learning. Characterizing the generalization error of a learning algorithm in terms of the mutual information between its input and output brings the significant advantage of the ability to incorporate the dependence on the data distribution, the hypothesis space, and the learning algorithm. This is in sharp contrast with conventional VC-dimension bounds and uniform stability bounds. Tighter mutual information bounds between the parameters and a single data point are explored in \citep{bu2020tightening}.  \citet{pensia2018generalization} applied the mutual-information framework to a broad class of iterative algorithms, including SGLD and stochastic gradient Hamiltonian Monte Carlo (SGHMC).  \citet{negrea2019information} provided data-dependent estimates of information-theoretic bounds for SGLD. For a recent comprehensive study, see \citet{steinke2020reasoning}.

\section{Preliminaries}
\paragraph{Basic Notations}
We use upper case letters, e.g. $X,Y$, to denote random variables and corresponding calligraphic letters $\mathcal{X},\mathcal{Y}$ to denote the sets which they are defined on. We denote as $P_X$, the marginal probability distribution of $X$. Given the Markov chain $X \rightarrow Y$, $P_{Y|X}$ denotes the conditional distribution or the Markov transition kernel.  $X \indep Y$ means $X$ and $Y$ are independent. 

And let us recall some basic definitions:

\begin{definition}
 Let $\psi_X(\lambda) \eqdef \log\mathbb{E}[e^{\lambda(X - \mathbb{E}[X])}]$ denote the cumulant generating function(CGF) of random variable $X$. Then $X$ is said to be $\sigma$-subgaussian if we have
\begin{equation*}
    \psi_X(\lambda) \leq \frac{\lambda^2\sigma^2}{2}, \forall\lambda \in \mathbb{R}\, .
\end{equation*}
\end{definition} \label{def:def1}

\begin{definition}
Let $X$, $Y$ and $Z$ be arbitrary random variables, and let $D_{\text{KL}}$ denote the KL divergence. The mutual information between $X$ and $Y$ is defined as:
\[
I(X;Y) \eqdef D_{\text{KL}}(P_{X,Y}||P_XP_Y)\,.
\]
The disintegrated mutual information between $X$ and $Y$ given $Z$ is defined as:
\[
I^{Z}(X;Y)\eqdef D_{\text{KL}}(P_{X,Y|Z}||P_{X|Z}P_{Y|Z})\, .
\]
The corresponding conditional mutual information is defined as:
\[
I(X;Y|Z) \eqdef \mathbb{E}_{Z}[I^{Z}(X;Y)]\,.
\]
\end{definition} \label{def:def2}

\paragraph{Information theoretic bound for single task learning} We consider an unknown distribution $\mu$ on an instance space $\mathcal{Z} = \mathcal{X}\times \mathcal{Y}$, and a set of independent samples $S = \{Z_i\}_{i=1}^{m}$ drawn from $\mu$: $Z_i \sim \mu$ and $S \sim \mu^{m}$.
Given a parametrized hypothesis space $\mathcal{W}$ and a loss function $\ell:\mathcal{W}\times \mathcal{Z} \rightarrow R$, the true risk and the empirical risk of $w\in\Wcal$ are respectively defined as $ R_\mu(w) \eqdef \mathbb{E}_{Z\sim \mu}\ell(w, Z)$ and  $R_S(w) \eqdef (1/m) \sum_{i=1}^m \ell(w, Z_i)$.

Following the setting of information-theoretic learning \citep{russo2019much,xu2017information,bu2020tightening}, a learning algorithm $\mathcal{A}$ is a randomized mapping that takes a dataset $S$ as input and outputs a hypothesis $W$ according to a conditional distribution $P_{W|S}$, \ie, $W=\mathcal{A}(S) \sim P_{W|S}$.\footnote{Note that the conditional distribution $P_{W|S}$ is different from the posterior distribution in Bayes learning.} The (mean) generalization error $\text{gen}(\mu, \mathcal{A})\eqdef \mathbb{E}_{W,S} [R_{\mu}(W) - R_S(W)]$ of an algorithm $\mathcal{A}$ is then bounded according to:

\begin{theorem} (\citet{xu2017information}) Suppose that for each $w \in \mathcal{W}$, the prediction loss $\ell(w, Z)$ is $\sigma$-subgaussian with respect to $Z \sim \mu$. Then for any randomized learner $\mathcal{A}$ characterized by $P_{W|S}$, for $S\sim\mu^m$, we have 
\[|\text{gen}(\mu, \mathcal{A})| \leq \sqrt{\frac{2\sigma^2}{m}I(W;S)}\, .\]
\end{theorem}
$I(W;S)$ is the mutual information between the input and output of algorithm $\Acal$ (see definition in Definition A.2). Theorem 3.1 reveals that the less the output hypothesis $W$ depends on the dataset $S$, the smaller the generalization error of the learning algorithm will be.

\section{Problem Setup}


Following \cite{baxter2000model}, we assume
that all tasks originate from a common \emph{environment} $\tau$, which is a probability
measure on the set of probability measures on $\mathcal{Z}=\mathcal{X}\times\mathcal{Y}$. The draw of $\mu\sim\tau$ represents encountering a learning task $\mu$ in the environment $\tau$. To run a learning algorithm for a task, we need to draw a set of data samples from $\mu$. In meta learning, there are multiple tasks, for simplicity, we assume that each task has the same sample size $m$. Based on \citet{maurer2016benefit}, the environment $\tau$ induces a mixture distribution $\mu_{m,\tau}$ on $\mathcal{Z}^m$ such that $\mu_{m, \tau}(A)=\mathbb{E}_{\mu \sim \tau}[\mu^m{(A)}], \forall A \subseteq \mathcal{Z}^m$. Thus the $m$ data points in $S$ that are independently sampled from a random task $\mu$ encountered in $\tau$ is denoted as $S \sim \mu_{m,\tau}$.

Consequently, for $n$ train tasks that are independently sampled from the environment $\tau$, each train data set is denoted as $S_i\sim \mu_{m, \tau}$ for $i\in [n]$. Analogously, for $k$ test tasks data sets, we denote $S^{\text{te}}_i\sim \mu_{m, \tau}$ for each $i\in [k]$.  We further denote the (full) training set as $S_{1:n}= (S_1,..., S_n)$ and the (full) testing set as $S^{\text{te}}_{1:k} = (S^{\text{te}}_1,...,S^{\text{te}}_k)$.


\subsection*{Meta Learner \& Base Learner} 

Since different tasks are assumed to be an i.i.d. sampling from $\tau$, they should share some common information. We use a \emph{meta parameter} $U\in\mathcal{U}$ to represent this shared knowledge. We also denote by $W_{1:n}= (W_1,\dots,W_n)$ the \emph{task specific parameters}, where each $W_i\in\mathcal{W}, \forall i\in[n]$. By exploring the relations between $U$ and $W$, we can design different meta learning algorithms. For example, \cite{pentina2014pac,amit2018meta} treated $U$ as the hyper-parameters of the base learner that produces $W$. In gradient based meta-learning such as MAML~\citep{finn2017model}, $U$ was chosen to be an initialization of $W$ (hence, $\mathcal{U}=\mathcal{W}$) for a gradient-descent base learner. 

We define the \emph{meta learner} $\mathcal{A}_{\text{meta}}$ as an algorithm that takes the data sets $S_{1:n}$ as input, and then outputs a random meta-parameter $U = \mathcal{A}_{\text{meta}}(S_{1:n}) \sim P_{U|S_{1:n}}$, which is a distribution that characterizes $\mathcal{A}_{\text{meta}}$. When learning a new task, the \textit{base learner} $\mathcal{A}_{\text{base}}$ uses a new data set $S\sim \mu_{m,\tau}$ and the estimated meta-parameter $U$ to output a stochastic predictor $W = \mathcal{A}_{\text{base}}(U,S) \sim P_{W|U,S}$.\footnote{Although the base learner is the same, $P_{W_i|U, S_i}$ is different for each task $i$ due to the different data set $S_i$.} 


To evaluate the quality of the meta information $U$ for learning a \emph{new task}, we define the \emph{true meta risk}, given the base learner $\mathcal{A}_{\text{base}}$, as

\hspace{4cm}$R_{\tau}(U) \eqdef \mathbb{E}_{S \sim \mu_{m,\tau}} \mathbb{E}_{W \sim P_{W|S, U}}[R_{\mu}(W)]\,.$ 

\subsection*{Joint Training \& Alternate Training}

Since $\tau$ and $\mu$ are unknown, we can only estimate $U$ and $W$ from the observed data. Generally, there are two different types of methods for evaluating meta and task parameters. 

\begin{figure}[ht]
\hspace{-1cm}
     \centering
     \begin{subfigure}[b]{0.45\textwidth}
         \centering
         \includegraphics[width=0.85\textwidth]{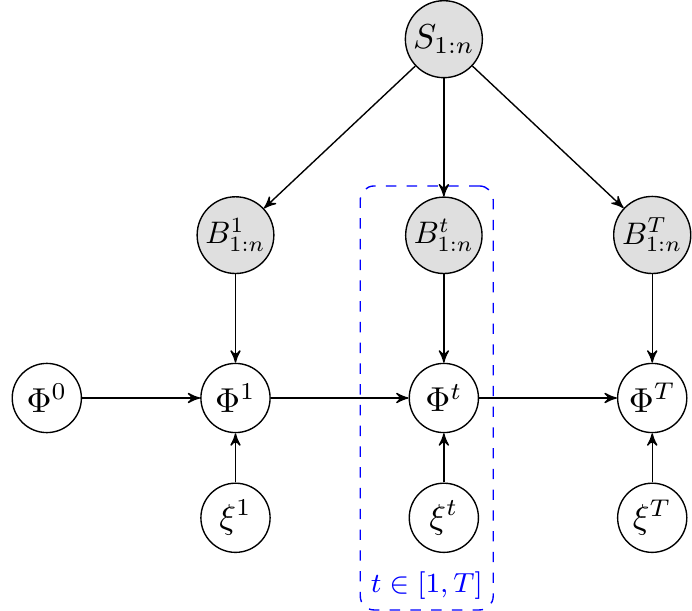}
         \caption{Joint Training}
         \label{fig:jlt}
     \end{subfigure}
     \quad\quad
     \begin{subfigure}[b]{0.4\textwidth}
         \centering
         \includegraphics[width=0.88\textwidth]{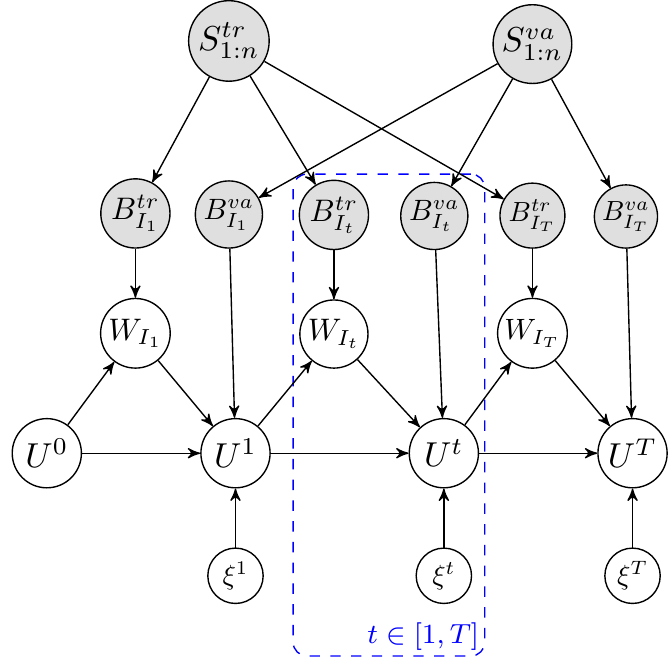}
         \caption{Alternate Training}
         \label{fig:alt}
     \end{subfigure}
     \caption{Parameter updating strategy through noisy iterative approach.}
\end{figure}

For \emph{Joint Training}~\citep{amit2018meta,pentina2014pac}, the whole dataset $S_{1:n}$ is used to jointly evaluate all the parameters $(U,W_{1:n})$ in parallel. A similar training protocol is illustrated in Fig.~\ref{fig:jlt}.
Then the corresponding \emph{empirical meta risk} w.r.t. $U$ is defined as:

\hspace{4cm}$R_{S_{1:n}}(U) \eqdef \frac{1}{n}\sum_{i=1}^n \mathbb{E}_{W_i \sim P_{W_i|S_i,U}}[R_{S_i}(W_i)].$

For \emph{Alternate training}, used in modern deep meta-learning algorithms~\citep{finn2017model}, $S_i$ is randomly split into two smaller datasets: a meta-train set $S^{\text{tr}}_i$ with $|S^{\text{tr}}_i|=m_{\text{tr}}$ and a meta-validation set $S^{\text{va}}_i$ with $|S^{\text{va}}_i| = m_{\text{va}}$ examples for each $i\in [n]$. In few-shot learning, $S_{1:n}^{\text{tr}}$ and $S_{1:n}^{\text{va}}$ are denoted as the support set and query set. Additionally, we have $m = m_{\text{tr}} + m_{\text{va}}$ and $S_i^{\text{tr}} \indep S_i^{\text{va}}$. An example of the training protocol is illustrated in Fig.~\ref{fig:alt}, where $(U,W_{1:n})$ are alternately updated through $S^{\text{va}}_{1:n}$ and $S^{\text{tr}}_{1:n}$, respectively. The corresponding \emph{empirical meta risk} w.r.t $U$ is defined as:

\hspace{4cm}$\tilde{R}_{S_{1:n}}(U)\eqdef\frac{1}{n}\sum_{i=1}^n \mathbb{E}_{W_i \sim P_{W_i|S^{\text{tr}}_i,U}}[R_{S^{\text{va}}_i}(W_i)]\,$

Then, the \emph{meta generalization error} within these two modes w.r.t. $\mathcal{A}_{\text{meta}}$ and $\mathcal{A}_{\text{base}}$ are respectively defined as 

\hspace{3cm} $\text{gen}^{\text{joi}}_{\text{meta}}(\tau, \mathcal{A}_{\text{meta}}, \mathcal{A}_{\text{base}})\ \eqdef\  \mathbb{E}_{U,S_{1:n}}[R_{\tau}(U) - R_{S_{1:n}}(U)],$

\hspace{3cm} $\text{gen}^{\text{alt}}_{\text{meta}}(\tau,\mathcal{A}_{\text{meta}}, \mathcal{A}_{\text{base}})\ \eqdef\ \mathbb{E}_{U,S_{1:n}}[R_{\tau}(U) - \tilde{R}_{S_{1:n}}(U)]\,.$

\section{Information-Theoretic Generalization Bounds}

We provide here novel generalization bounds for joint and alternate training, which are respectively characterized by mutual information (MI) and conditional mutual information (CMI). These theoretical results are valid for any \emph{randomized} algorithm $\mathcal{A}_{\text{meta}}$ and $\mathcal{A}_{\text{base}}$. But for some deterministic algorithms producing deterministic predictors, the mutual information bound can be vacuous.


\subsection{Mutual Information (MI) Bound in Joint Training}
\begin{theorem}\label{thm:mi-jt}
Suppose all tasks use the same loss $\ell(Z,w)$, which is $\sigma$-subgaussian for each $w \in \mathcal{W}$, where $Z \sim \mu, \mu \sim \tau$.  Then, the meta generalization error for joint training is upper bounded by 
\begin{equation*}\label{eq:eq_mi_joi}
 |\text{gen}^{\text{joi}}_{\text{meta}}(\tau,\mathcal{A}_{\text{meta}}, \mathcal{A}_{\text{base}})| \leq  \sqrt{\frac{2\sigma^2}{nm} I(U,W_{1:n};S_{1:n})}\, .
\end{equation*}
\end{theorem}
The proof of Theorem \ref{eq:eq_mi_joi} is presented in Appendix \ref{proof_thm_mi}. Moreover, according to the chain rule of mutual-information, the error bound in Theorem \ref{thm:mi-jt} can be further decomposed as

\hspace{0.5cm}$\sqrt{\frac{2\sigma^2}{mn} \left(I(U;S_{1:n}) + \sum_{i=1}^n I(W_{i}; S_{i}|U)\right)} \leq \sqrt{\frac{2\sigma^2}{mn} I(U;S_{1:n})} + \sqrt{\frac{2\sigma^2}{mn} \sum_{i=1}^n I(W_i;S_i|U)}$\, . 

\textbf{Discussions}  The first and second terms reflect, respectively, the environmental and task-level uncertainty. \textbf{1)} In the limit of a very large number of tasks ($n\to\infty$) and a finite number $m$ of samples per task, the first term converges to zero, while the second term remains non-zero. This is consistent with Theorem 1 of \citet{bai2021important}, where they proved that joint training has a bias in general. However, this non-zero term will be smaller than the mutual information of single-task learning. Indeed, let $I(W;S)$ denotes the mutual information of single-task learning, we have, as shown in Appendix \ref{Proof_ben}, that   
$I(W;S) \geq I(W;S|U) \approx \frac{1}{n} \sum_{i=1}^n I(W_i;S_i|U)$, which illustrates the benefits of learning the meta-parameter $U$. \textbf{2)} When we have a constant number $n$ of tasks, while the number $m$ of samples per task goes to infinity, the whole bound will converge to zero. Note that the meta generalization error bound reflects how the meta-information assists a new task to learn. If the new task has a sufficiently large number $m$ of samples, the generalization error will be small, and the meta-information $U$ does not significantly help learning the new task.

\textbf{Relation with previous work} Since mutual information implicitly depends on the unknown distribution $\tau$, it is hard to estimate and minimize~\citep{mcallester2020formal}. By introducing an arbitrary distribution-free prior $Q$ on $\mathcal{U}\times\mathcal{W}^n$,  we can upper bound
$I(U,W_{1:n};S_{1:n}) \leq I(U,W_{1:n};S_{1:n}) + D_{\text{KL}}(P_{U,W_{1:n}}||Q) = \mathbb{E}_{S_{1:n}}D_{\text{KL}}(P_{U,W_{1:n}|S_{1:n}}||Q)$ 
(see Lemma \ref{Lemma_B.1}.1). If we set a joint prior $Q = \mathcal{P}\times\prod_{i=1}^n P$, then the bound of Theorem~\ref{thm:mi-jt} becomes similar to the one proposed by~\cite{amit2018meta}, where $\mathcal{P}$ is the hyper-prior and $P$ is the task-prior in their settings. Finally, note that the bound of Theorem 5.1 is tighter than the one proposed by~\citep{pentina2014pac}, where the $\text{KL}$ divergence is outside of the square root function.

\subsection{Conditional Mutual Information (CMI) Bound for Alternate Training}
\begin{theorem}\label{thm:cmi-al}
Assume that all the tasks use the same loss function $\ell(Z,w)$, which is $\sigma$-subgaussian for each $w \in \mathcal{W}$, where $Z\sim\mu, \mu\sim\tau$.Then we have
\[
 |\text{gen}^{\text{alt}}_{\text{meta}}(\tau, \mathcal{A}_{\text{meta}}, \mathcal{A}_{\text{base}})| \leq \mathbb{E}_{S^{\text{tr}}_{1:n}}\sqrt{\frac{2\sigma^2I^{S^{\text{tr}}_{1:n}}(U, W_{1:n};S^{\text{va}}_{1:n})}{nm_{\text{va}}}}
\leq\sqrt{\frac{2\sigma^2I(U,W_{1:n};S^{va}_{1:n}|S^{tr}_{1:n})}{nm_{\text{va}}}}\, .
\]
\end{theorem}

See the proof in Appendix \ref{proof_thm_cmi}.
The second inequality is obtained with the Jensen's inequality for the concave square root function and Lemma \ref{Lemma_B.3}.3. Additionally, we can apply the chain rule on the conditional mutual information, to obtain the following decomposition:

$I(U,W_{1:n};S^{\text{va}}_{1:n}|S^{\text{tr}}_{1:n}) = I(U;S^{\text{va}}_{1:n}|S^{\text{tr}}_{1:n}) + \sum_{i=1}^n I(W_{i};S^{\text{va}}_{i}|U,S^{tr}_{i})$

\hspace{3.1cm} $= \mathbb{E}_{S_{1:n}}D_{\text{KL}}(P_{U|S_{1:n}}||P_{U|S^{\text{tr}}_{1:n}})
+ \mathbb{E}_{U,S_{1:n}}\sum_{i=1}^{n}D_{\text{KL}}(P_{W_i|S_i,U}||P_{W_i|S^{\text{tr}}_i,U}))$.

\textbf{Discussions}~The aforementioned decomposition reveals the following intuition: Suppose the outputs of the base learner and meta learner w.r.t. different input data-sets are similar (\ie~the learning algorithms are stable w.r.t. the data). In that case, the meta-generalization error bound will be small.
Moreover, since the bound of Theorem~\ref{thm:cmi-al} is data-dependent w.r.t. $S_{1:n}^{tr}$, we can obtain tighter theoretical results through these data-dependent estimates. This is to be contrasted with the mutual information bound of Theorem~\ref{thm:mi-jt}, which depends on the unknown distribution and can thus be inflated through the variational form. In Sec~\ref{algo}, we analyze noisy iterative algorithms in deep few-shot learning to obtain tighter estimates. Besides, there exists an inherent trade-off in  choosing $m_{\text{va}}$. If $m_{\text{va}}$ is large, then the denominator in the bound is large. However, since $m_{\text{va}} = m - m_{\text{tr}}$, $D_{\text{KL}}(P_{U|S_{1:n}}||P_{U|S^{\text{tr}}_{1:n}})$ and $D_{\text{KL}}(P_{W_i|S_i,U}||P_{W_i|S^{\text{tr}}_i,U})$ will also become large since smaller $m_{\text{tr}}$ will lead to less reliable outputs.

\section{Generalization Bounds for Noisy Iterative Algorithms}\label{algo}
We will now exploit Theorem~\ref{thm:mi-jt} and \ref{thm:cmi-al}. to analyze concrete algorithms. Specifically, in noisy iterative algorithms, all the iterations are related through a Markov structure, which can naturally apply the information chain rule. Our theoretical results focus on one popular instance: SGLD \citep{welling2011bayesian}, which is a variant of Stochastic Gradient Descent (SGD) with the addition of a scaled isotropic Gaussian noise to each gradient step. The injected noise allows SGLD to escape the local minima and asymptotically converge to global minimum for sufficiently regular non-convex objectives~\citep{raginsky2017non}. 
It is worth mentioning that other types of iterative algorithms such as SG-HMC can be also analyzed within our theoretical framework, which is left as the future work.

Since the algorithms to be analyzed require sampling mini-batches of sample at each iteration, we make the following independence assumption:

\textbf{Assumption 1} \textit{The sampling strategy is independent of the parameters and the previous samplings.} 

\subsection{Bound for Joint Training with Bounded Gradient}
In joint training, the meta and base parameters are updated simultaneously. We denote $\Phi \eqdef (U,W_{1:n}) \in \mathcal{U}\times\mathcal{W}^n$, $\mathcal{U}\subseteq \Reals^k, \mathcal{W}\subseteq \Reals^d$. The training strategy is illustrated in Fig~\ref{fig:jlt}.
Concretely, the learning algorithm executes $T$ iterations. We further denote $\Phi^t$ as the updated parameter at iteration $t\in[T]$, with $\Phi^0$ being a random initialization. 

At iteration $t\in[T]$ and for task $i \in [n]$, we randomly sample a batch $B_i^t \subseteq S_i$ of size $b$ and an isotropic Gaussian noise $\xi^t\sim N(0,\sigma_t^2\mathbb{I}_{(nd +k)})$. Let $\xi^t = (\xi^t_0, \dots, \xi^t_n)$, where $\xi^t_0 \in \mathcal{U}$, and $\xi^t_{i}\in\Wcal,~\forall i\in[n]$. Then the updating rule at iteration $t$ can be expressed as 

\hspace{4cm} $\Phi^t = \Phi^{t-1} -\eta_t G(\Phi^{t-1}, B^t_{1:n}) + \xi^t\, ,$

where $G$ is the gradient of the empirical meta-risk on $B^t_{1:n}$ w.r.t. all the parameters $\Phi$, and where $\eta_t$ is the learning rate. In addition, we assume bounded gradients:

\textbf{Assumption 2} 
\textit{The gradients are bounded, \ie, $\sup\limits_{\Phi\in R^{(nd +k)}, s\in\mathcal{Z}^{bn}}||G(\Phi,s)||_2\leq L,$ with $L > 0.$}

Then the mutual information in Theorem~\ref{thm:mi-jt} can be upper-bounded as follows.
\begin{theorem}
Based on Theorem \ref{thm:mi-jt}, for the SGLD algorithm that satisfies Assumptions 1 \& 2, the mutual information for joint training satisfies

\hspace{4cm}$I(\Phi;S_{1:n}) \leq \sum_{t=1}^T \frac{nd+k}{2} \log(1 + \frac{\eta_t^2L^2}{(nd + k)\sigma_t^2})\, .$

Specifically, if $\sigma_t=\sqrt{\eta_t}$, and $\eta_t = \frac{c}{t}$ for $c>0$, we have:

\hspace{4cm}$|\text{gen}_{\text{meta}}^{\text{joi}}(\tau, \mathcal{A}_{meta}, \mathcal{A}_{base})| \leq \frac{\sigma L}{\sqrt{nm}}\sqrt{c\log T + c}\, .$

\end{theorem}

See the proof in Appendix \ref{proof_thm_joint_sgld}. It is worth mentioning that \citet{amit2018meta} used $\mathcal{U} \subseteq \Reals^{2d}, \mathcal{W} \subseteq \Reals^{d}$. They adopt a similar variational form of our mutual information bound, where they use a factorized Gaussian $Q_{\theta} = \mathcal{N}(\theta, \mathbb{I}_{2d})$ to approximate $P_{U|S_{1:n}}$ and $Q_{\phi_i} = \mathcal{N}(\mu_i, \sigma^2_i)$ to approximate $P_{W_i|U,S_i}$. They set $P_U = \mathcal{N}(0, \mathbb{I}_{2d})$, and $P_{W_i|U} = \mathcal{N}(\mu_P, \sigma^2_P)$, where $(\mu_P, \sigma^2_P) \sim Q_{\theta}$. Then they optimize the meta empirical risk plus the bound w.r.t. the parameters $\theta \in \Reals^{2d}$ and $(\mu_i, \sigma^2_i) \in \Reals^{2d}$ by SGD. Our method is different since we do not use parametric approximations. Instead, we simulate the joint distribution with SGLD.


\subsection{Bound for Alternate Training with Gradients Incoherence}

The updating strategy for alternate training is illustrated in Fig.~\ref{fig:alt}. We also use SGLD for the meta learner $\mathcal{A}_{\text{meta}}$ and base learner $\mathcal{A}_{\text{base}}$, and denote this algorithm by Meta-SGLD. To build a connection with MAML, we consider the scenarios with $\mathcal{U} = \mathcal{W} \subseteq \Reals^d$, where meta-parameter $U$ is a common initialization for the task parameters $W_{1:n}$ to achieve a fast adaptation.

The Meta-SGLD algorithm has a nested loop structure: the outer loop includes $T$ iterations of SGLD for updating the meta-parameters $U$; at each outer loop iteration $t\in [T]$, there exists several parallel inner loops, where each loop is a $K$-iteration SGLD to update different task-specific parameters $W_{i}$.

\textbf{Outer Loop Updates}  

It is computationally expensive to learn meta information from all the tasks when the number $n$ of tasks is large---a common situation in few-shot learning. Thus, for each $t\in[T]$, we sample a mini-batch of tasks that are indexed by $I_{t} \subseteq [n]$.  Then the corresponding meta-train and meta-validation data sets are denoted as $B^{\text{tr}}_{I_{t}}$ and $B^{\text{va}}_{I_{t}}$, respectively. The task specific parameters are denoted as $W_{I_t}=\{W_i: i\in I_t\}$.
In addition, an isotropic Gaussian noise $\xi^t\sim \mathcal{N}(\mathbf{0},\sigma_t^2\mathbb{I}_{d})$ is also injected during the update. Then, the update rule w.r.t. $U$ is expressed as:

\hspace{4cm} $U^t = U^{t-1} - \eta_t \nabla \tilde{R}_{B^{\text{va}}_{I_t}}(U^{t-1}) + \xi^t\, ,$

where $\tilde{R}_{B^{\text{va}}_{I_t}}(U^{t-1}) = \frac{1}{|I_t|}\sum_{i\in I_t} \mathbb{E}_{W_i \sim P_{W_i|B^{\text{tr}}_{i,t}, U^{t-1}}}[R_{B^{\text{va}}_{i,t}}(W_i)]$ is the empirical meta risk evaluated on $B^{\text{va}}_{I_t}$, and $\eta_{t}$ is the meta learning rate at $t$.
In addition, we denote the \emph{gradient incoherence} of meta parameter $U$ at iteration $t$ as $\epsilon_t^u \eqdef \nabla\tilde{R}_{B_{I_t}}(U^{t-1}) - \nabla\tilde{R}_{B^{\text{tr}}_{I_t}}(U^{t-1})$.

\textbf{Inner Loop Updates} 


Given the outer loop iteration $t$, for each inner iteration $k\in[K]$, we randomly sample a batch of data for task $i \in I_t$ from $B^{\text{tr}}_{i,t}$ (the $i$-th task in $B^{\text{tr}}_{I_t}$), which is denoted as $B^{\text{tr}}_{i,t,k}$. Then the update rules for the task parameters can be formulated as:

\hspace{3cm} $W^{0}_{i,t} = U^{t-1}\,, W_{i,t}^k = W^{k-1}_{i,t} - \beta_{t,k} \nabla R_{B^{\text{tr}}_{i,t,k}}( W^{k-1}_{i,t}) + \zeta^{t,k}\, ,$

where $\beta_{t,k}$ is the learning rate for task parameter, $\zeta^{t,k} \sim \mathcal{N}(\mathbf{0},\sigma_{t,k}^2\mathbb{I}_{d})$ is the injected isotropic Gaussian noise (not shown in the figure) at inner iteration $k$. Analogously, we can compute the gradient incoherence w.r.t. the task parameters $W_{i,t}^k$. We first sample a batch $B_{i,t,k}$ from $B^{\text{va}}_{i,t}\bigcup B^{\text{tr}}_{i,t}$ (the union of training and validation task batches). Then the gradient incoherence of task specific parameters at the $k$-th inner update, task $i$, and outer iteration $t$ is defined as: $\epsilon_{t,i,k}^w \eqdef \nabla R_{B_{i,t,k}}(W^{k-1}_{i,t}) - \nabla R_{B^{\text{tr}}_{i,t,k}}(W^{k-1}_{i,t})$.

\emph{Relation to MAML} \hspace{0.4cm} Without the noise injection, the whole updating protocol described above is exactly MAML. Specifically, if we set $K=1$, the empirical meta loss can be expressed as:\hspace{0.8cm}$\tilde{R}_{B^{\text{va}}_{I_t}}(U^{t-1}) = \frac{1}{|I_t|}\sum_{i\in I_t} R_{B^{\text{va}}_{i,t}}(W^{K}_{i,t}) = \frac{1}{|I_t|}\sum_{i\in I_t} R_{B^{\text{va}}_{i,t}}(U^{t-1} -\beta_{t,1} \nabla R_{B^{\text{tr}}_{i,t,1}}(U^{t-1}))\,,$
and $\frac{1}{|I_t|}\sum_{i\in I_t}\nabla R_{B^{\text{va}}_{i,t}}(W_{i,t}^K)$ is the first-order MAML gradient.

Based on the nested loop structure and the independent sampling strategy, we have the following data-dependent generalization-error bound for Meta-SGLD.

\begin{theorem}
Based on Theorem \ref{thm:cmi-al}, for the Meta-SGLD that satisfies Assumption 1, if we set $\sigma_t=\sqrt{2\eta_t/\gamma_t}$, $\sigma_{t,k}=\sqrt{2\beta_{t,k}/\gamma_{t,k}}$, where $\gamma_t$ and $\gamma_{t,k}$ are the inverse temperatures. The meta generalization error for alternate training satisfies 
\[
|\text{gen}^{\text{alt}}_{\text{meta}}(\tau, \text{SGLD}, \text{SGLD})| \leq 
\sqrt{\frac{2\sigma^2I(U,W_{1:n};S^{va}_{1:n}|S^{tr}_{1:n})}{nm_{\text{va}}}}\ 
 \leq\ \frac{\sigma}{\sqrt{nm_{\text{va}}}}\sqrt{ \epsilon_U + \epsilon_W}\, ,
\]
where 
\[\epsilon_U = \sum_{t=1}^T \mathbb{E}_{B^{va}_{I_t},B^{tr}_{I_t},W_{I_t},U^{t-1}} \frac{\eta_t\gamma_t\|\epsilon_t^u\|^2_2}{2},~~~~\epsilon_W=\sum_{t=1}^T\sum_{i=1}^{|I_t|} \sum_{k=1}^K \mathbb{E}_{{B^{va}_{i,t,k},B^{tr}_{i,t,k}, W_{i,t}^{k-1}}} \frac{\beta_{t,k}\gamma_{t,k}\|\epsilon_{t,i,k}^w\|^2_2}{2}\, .\]
\end{theorem}

The proof is provided in Appendix~\ref{proof_thm_alt_sgld}. The bound of Theorem 6.2 consists of two parts: $\epsilon_U$, which reflects the generalization error bound of the meta learner, and $\epsilon_W$, which reflects the generalization bound of the base learner. Moreover, $\epsilon_U$ and $\epsilon_W$ are characterized by the accumulated \emph{gradient incoherence} and predefined constants such as learning rates, inverse temperatures, and number of iterations. Compared with previous works such as \citep{denevi2019learning,finn2019online}, Theorem 6.2 exploits the gradient difference between two batches rather than the Lipschitz constant of the loss function (and, consequently, its tighter estimation, the gradient norm of the empirical meta-risk and the individual task risks). This can give a more realistic generalization error bound since the Lipschitz constant for neural networks is often very large~\citep{scaman2018lipschitz}. In contrast, our empirical results reveal (see Sec~\ref{emp_valid}) that the gradient incoherence can be much smaller than the gradient norm on average.
Note that we have chosen fixed and large inverse temperatures to ensure small injected noise variance from the beginning of training. In addition, the step sizes also affect the bound w.r.t. training iteration numbers $T,K$. For example, assuming that the gradient incoherence is bounded, if we choose $\eta_t = \frac{1}{t}, \beta_{t,k} = \frac{1}{tk}$, the meta generalization error bound is in $\mathcal{O}(\sqrt{c_1\log T + c_2\log K})$, where $c_1, c_2$ are some constants. In contrast, when learning rates are fixed, the bound is in $\mathcal{O}(\sqrt{c_1 T + c_2 TK})$. 


\section{Empirical Validations}\label{emp_valid}
We validate Theorem 6.2 on both synthetic and real data. The numerical results demonstrate that, in most situations, the \emph{gradient incoherence} based bound is orders of magnitude tighter than the conventional meta learning bounds with the Lipschitz assumption, which is estimated with gradient norms.\footnote{Code is available at: \url{https://github.com/livreQ/meta-sgld}.}

\subsection{Synthetic Data}
We consider a simple example of 2D mean estimation to illustrate the meta-learning setup. We assume that the environment $\tau$ is a truncated 2D Gaussian distribution $\mathcal{N}((-4,-4)^T, 5\mathbb{I}_{2})$. A new task is also defined as a 2D Gaussian $\mathcal{N}(\mu, 0.1\mathbb{I}_{2})$ with $\mu \sim \tau$. To generate few-shot tasks, we sample $n=20000$ tasks from the environment with $\mu_i \sim \tau, \forall i\in[n]$. After sampling $\mu_i$ for each task, we further sample $m=16$ data points from $\mathcal{N}(\mu_i, 0.1\mathbb{I}_{2})$. At each iteration $t$, we randomly choose a subset of $5$ tasks ($|I_t|=5$) from the whole data set. We evaluate on three different few-shot settings with $m_{\text{va}}=\{1,8,15\}$ and the corresponding train size $m_{\text{tr}}=\{15,8,1\}$.  The detailed experiment setting is in Appendix~\ref{toy_exp}. 

The estimated meta-generalization upper bounds are shown in Fig.~\ref{fig:toy_data}. For a better understanding of the generalization behaviour w.r.t. the meta learner and base learner, we separately show the estimated bounds of $\sigma\sqrt{\frac{\epsilon_U}{nm_{\text{va}}}}$ (Fig.~\ref{fig:toy_data}(a)) and $\sigma\sqrt{\frac{\epsilon_W}{nm_{\text{va}}}}$ (Fig.~\ref{fig:toy_data}(b)). The expectation terms within the bound are estimated via Monte-Carlo sampling. To compare the conventional Lipschiz bound with ours, we approximately calculated a tighter estimation of the bound with the expected gradient norm at each iteration \cite{li2019generalization} instead of using a fixed Lipschitz constant, which is extremely vacuous in deep learning. The other components remain the same as the gradient incoherence bound.

The results on the synthetic data set reveal a substantial theoretical benefit compared with the conventional Lipschitz bound. Specifically, the magnitude of the bound is improved by a factor of 10 to 100. Interestingly, the gap between the gradient-norm and the gradient-incoherence bound is smallest when $m_{\text{va}}=15$. These theoretical results reveal that the generalization bound is unavoidably large if the base learner is trained on extremely few data (\eg, a 1-shot scenario). Since too few train data (small $m_{\text{tr}}$) induces high randomness and large instability in each training task. 

To further validate Theorem 6.2, We calculated the actual generalization gap by evaluating the expected difference between the train loss and test loss for the above mentioned tree settings. The actual generalization gap of $m_{\text{tr}}=1$ is also much larger compared to the other two setting, which also demonstrated the instability for extreme few shot learning (See Table \ref{tab:toy_8_8}, \ref{tab:toy_15_1} and \ref{tab:toy_1_15} in Appendix \ref{toy_add}).

\begin{figure}[ht]
    \centering
    \begin{subfigure}{0.45\textwidth}
        \centering
        \includegraphics[scale=0.45]{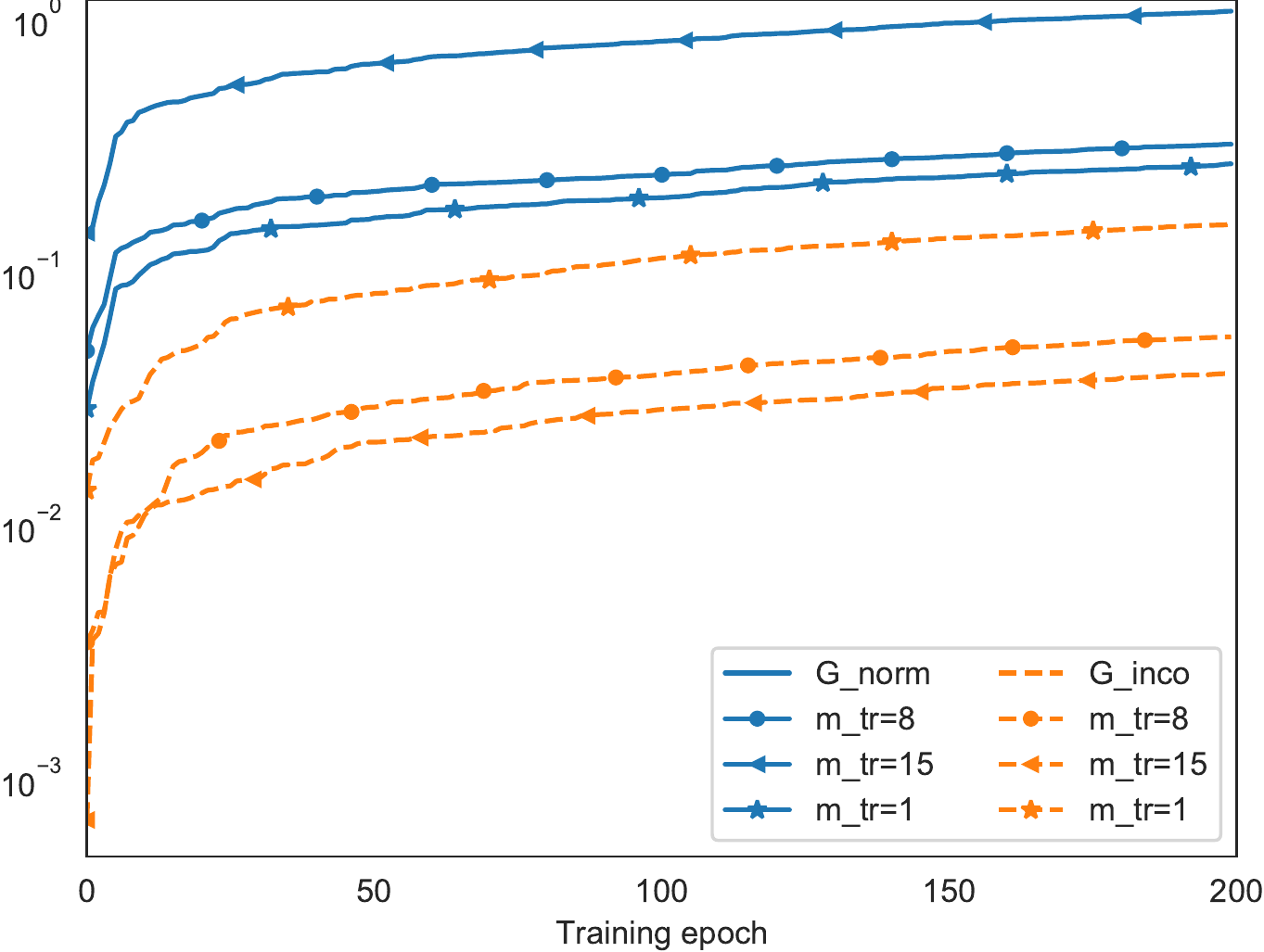}
        \caption{Bound of U}
    \end{subfigure}
    \hfill
    \begin{subfigure}{0.45\textwidth}
        \centering
        \includegraphics[scale=0.45]{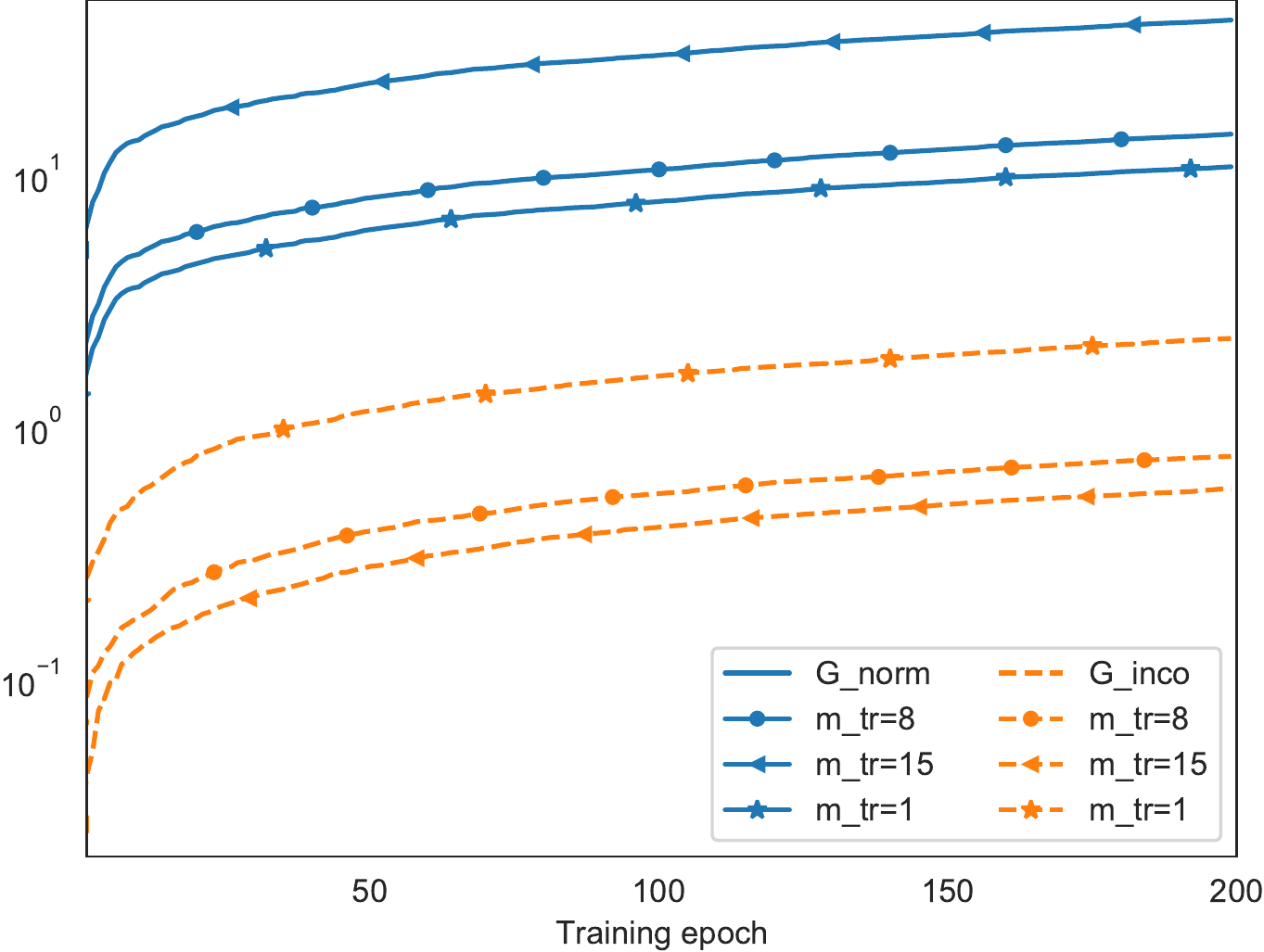}
        \caption{Bound of W}
    \end{subfigure}
    \caption{Synthetic data: Estimation of the generalization-error bound during the training ($T=200$). \emph{Left} Estimated error bound w.r.t. meta-learner. \emph{Right} The estimated error bound w.r.t. the base-learner. The curves in \textcolor{mBlue}{Blue} solid line and \textcolor{mYellow}{Orange} dashed line represent the estimated bound through gradient-norm (G\_Norm) and gradient-incoherence (G\_Inco) in different few-shot settings.}
    \label{fig:toy_data}
\end{figure}

\subsection{Few-Shot Benchmark}

To evaluate the proposed bound in modern deep few-shot learning scenarios, we have tested the Meta-SGLD algorithm on the Omniglot dataset \citep{lake2011one}. 
The Omniglot dataset contains 1623 characters for 50 different alphabets, and each character is present in 20 instances. 
We followed the experimental protocol of \citep{vinyals2016matching,finn2017model}, which aims to learn a N-way classification task for 1-shot or 5-shot learning. In our experiment, we conducted a 5-way classification learning. A train task consists of five classes (characters) randomly chosen from the first 1200 characters, each class has $m=16$ samples selected from the 20 instances. Similarly, a test task contains five classes randomly sampled from the rest 423 characters. Therefore, the meta train set has $n=\tbinom{1200}{5}$ tasks. At each epoch, we have trained the model with $|I_t|=32$ tasks. Analogous to the simulated data, we have conducted our experiment with $m_{\text{tr}}=\{15,8,1\}$ and $m_{\text{va}}=\{1,8,15\}$ and separately visualized the two components of the bound. The detailed experimental setting is provided in Appendix \ref{omniglot}.

The estimated bounds are shown in Fig.~\ref{fig:real_data}. Analogous to the results on synthetic data, the estimated error bound trough gradient-incoherence is \emph{tighter} than the gradient-norm based bound when $m_{\text{tr}}=8,15$. In particular, the gradient-incoherence bound w.r.t. $U$ is much tighter than the gradient-norm bound when $m_{\text{tr}}=15$, which illustrates the benefits of the proposed theory. Simultaneously, the gradient-incoherence bound is similar to the gradient-norm bound when $m_{\text{tr}}=1$, illustrating a theoretical limitation of learning with very few meta-train samples. Moreover, we observe that the optimal values for $m_{\rm{va}}$ depends on the environment since the tightest bound for Omniglot is achieved with $m_{\text{va}}=8$, which is different from what we have found for the synthetic data. 

Finally, we observed that the component of the generalization error bound that originates from task-specific parameters is numerically larger than the one the originates from the meta parameter, has compared to the results for simulated data. This perhaps illustrates an inherent difficulty in learning few-shot tasks with high-dimensional and complex data sets, where estimating the generalization error bound is apparently more challenging. Additional experimental results for test accuracy comparison with MAML on the aforementioned tree settings are presented in Appendix \ref{omniglot_add} Table \ref{tab:acc}. Comparison of bound values with the observed generalization error is also included (See Table \ref{tab:omg_8_8},\ref{tab:omg_15_1} and \ref{tab:omg_1_15}). We believe the less evident improvement with gradient incoherence bound compared to Synthetic data can be ascribed to the utilization of Batch Normalization. 

\begin{figure}[t]
    \centering
    \begin{subfigure}{0.48\textwidth}
        \centering
        \includegraphics[scale=0.45]{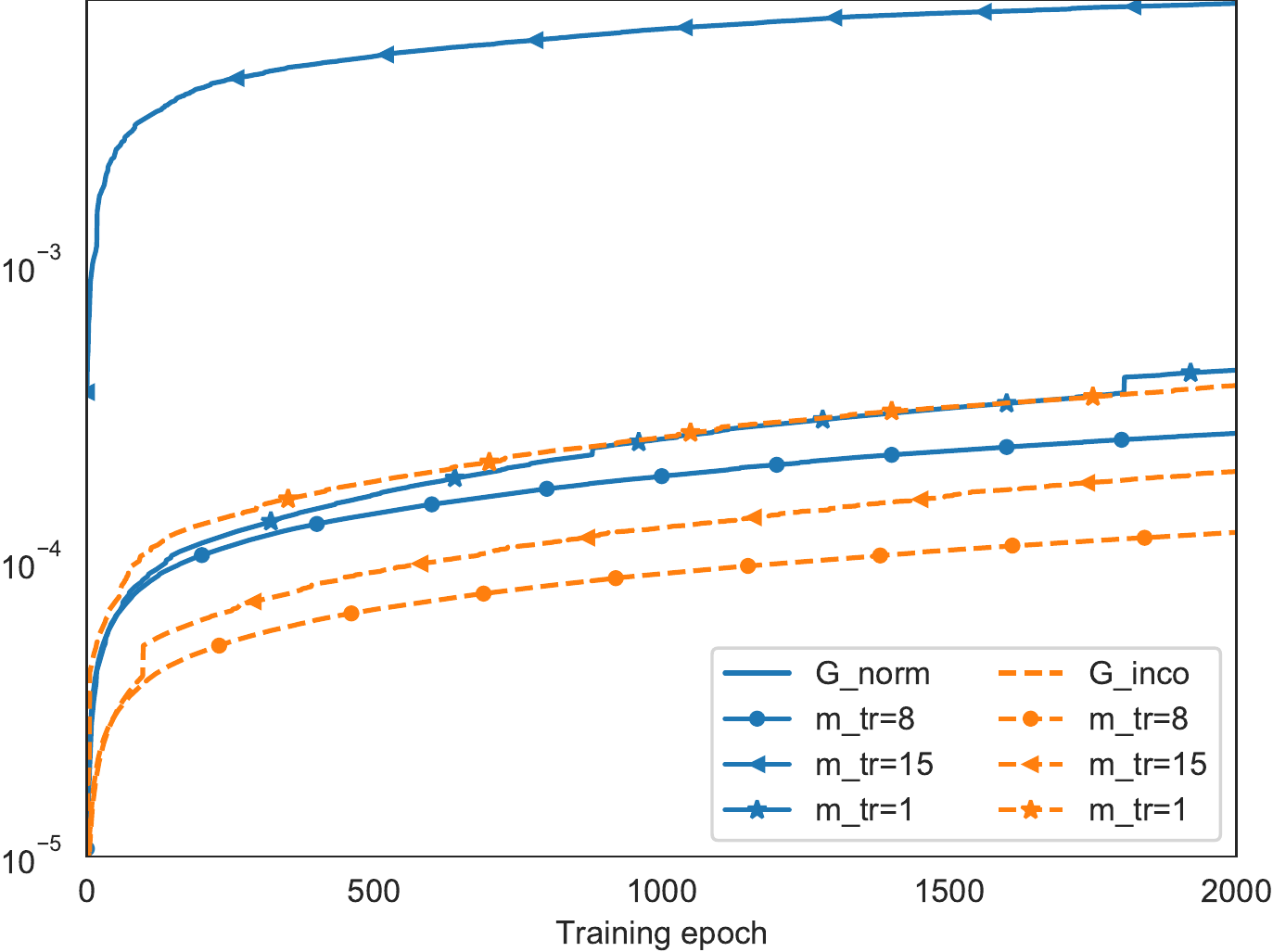}
        \caption{Bound of U}
    \end{subfigure}
    \hfill
    \begin{subfigure}{0.48\textwidth}
        \centering
        \includegraphics[scale=0.45]{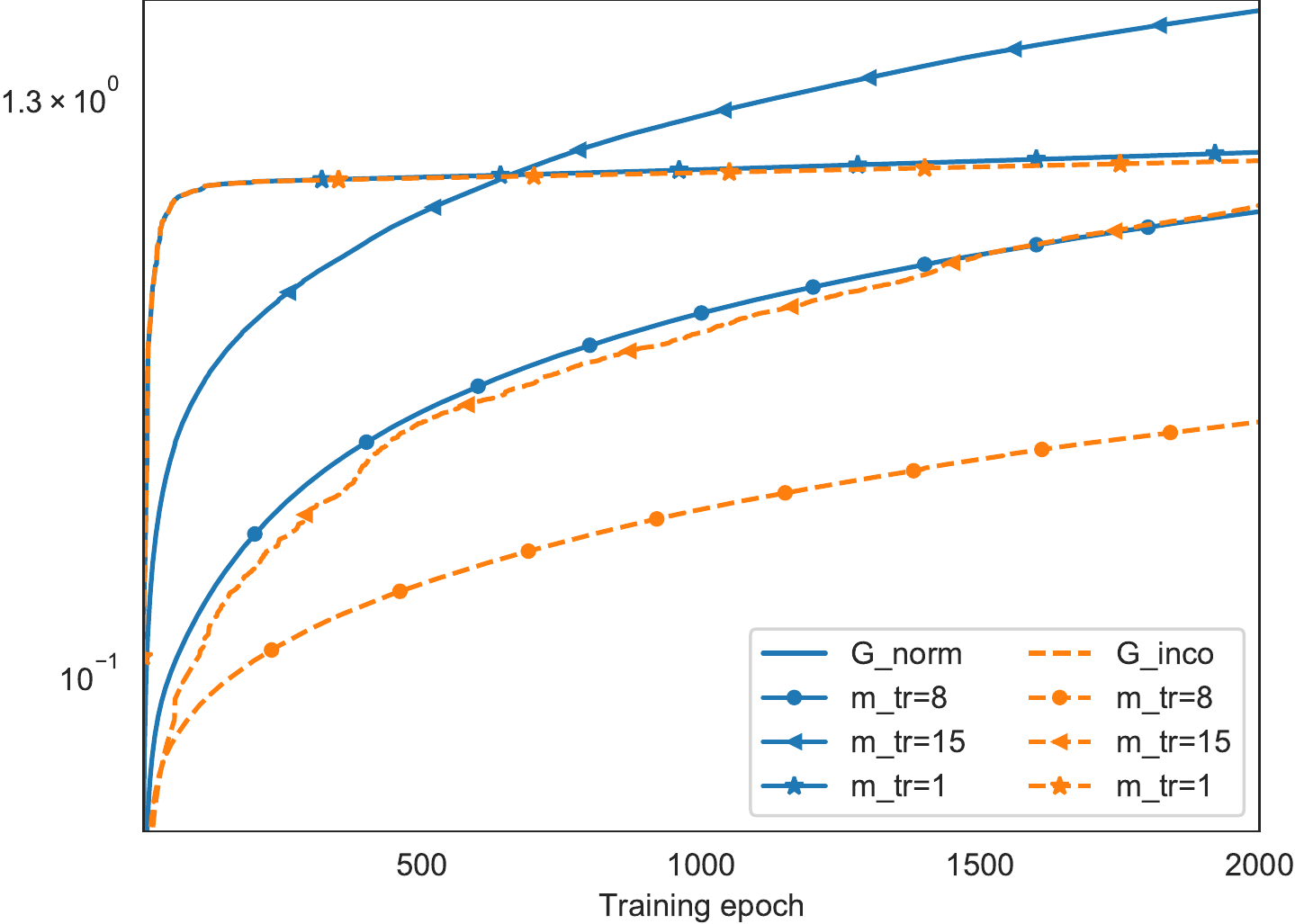}
        \caption{Bound of W}
    \end{subfigure}
    \caption{Omniglot: Estimation of the generalization-error bound during the training ($T=2000$). The curves in \textcolor{mBlue}{Blue} solid line and \textcolor{mYellow}{Orange} dashed line represent the estimated bound through gradient-norm (G\_Norm) and gradient-incoherence (G\_Inco) in different few-shot settings}
    \label{fig:real_data}
\end{figure}

\section{Conclusion}
We derived a novel information-theoretic analysis of the generalization property of meta-learning and provided algorithm-dependent generalization error bounds for both joint training and alternate training. Compared to previous gradient-based bounds that depend on the square norm of gradients, empirical validations on both simulated data and a few-shot benchmark show that the proposed bound is orders of magnitude tighter in most situations. Finally, we think that these theoretical results can inspire new algorithms through a deeper exploration of the relation between meta-parameters and task-parameters.  





\acksection
Work partly supported by NSERC Discovery Grant RGPIN-2016-05942 and the China Scholarship Council. We also thank SSQ Assurances and NSERC for their financial support through the Collaborative Research and Development Grant CRDPJ 529584 - 18.

\bibliographystyle{unsrtnat}
\bibliography{meta}

\begin{thebibliography}{49}
\providecommand{\natexlab}[1]{#1}
\providecommand{\url}[1]{\texttt{#1}}
\expandafter\ifx\csname urlstyle\endcsname\relax
  \providecommand{\doi}[1]{doi: #1}\else
  \providecommand{\doi}{doi: \begingroup \urlstyle{rm}\Url}\fi

\bibitem[Amit and Meir(2018)]{amit2018meta}
Ron Amit and Ron Meir.
\newblock Meta-learning by adjusting priors based on extended pac-bayes theory.
\newblock In \emph{International Conference on Machine Learning}, pages
  205--214, 2018.

\bibitem[Finn et~al.(2017)Finn, Abbeel, and Levine]{finn2017model}
Chelsea Finn, Pieter Abbeel, and Sergey Levine.
\newblock Model-agnostic meta-learning for fast adaptation of deep networks.
\newblock In \emph{ICML}, 2017.

\bibitem[Thrun and Pratt(1998)]{thrun1998learning}
Sebastian Thrun and Lorien Pratt.
\newblock Learning to learn: Introduction and overview.
\newblock In \emph{Learning to learn}, pages 3--17. Springer, 1998.

\bibitem[Liu et~al.(2018)Liu, Simonyan, and Yang]{liu2018darts}
Hanxiao Liu, Karen Simonyan, and Yiming Yang.
\newblock Darts: Differentiable architecture search.
\newblock \emph{arXiv preprint arXiv:1806.09055}, 2018.

\bibitem[Brown et~al.(2020)Brown, Mann, Ryder, Subbiah, Kaplan, Dhariwal,
  Neelakantan, Shyam, Sastry, Askell, et~al.]{brown2020language}
Tom~B Brown, Benjamin Mann, Nick Ryder, Melanie Subbiah, Jared Kaplan, Prafulla
  Dhariwal, Arvind Neelakantan, Pranav Shyam, Girish Sastry, Amanda Askell,
  et~al.
\newblock Language models are few-shot learners.
\newblock \emph{arXiv preprint arXiv:2005.14165}, 2020.

\bibitem[Garcia and Bruna(2017)]{garcia2017few}
Victor Garcia and Joan Bruna.
\newblock Few-shot learning with graph neural networks.
\newblock \emph{arXiv preprint arXiv:1711.04043}, 2017.

\bibitem[Ravi and Beatson(2018)]{ravi2018amortized}
Sachin Ravi and Alex Beatson.
\newblock Amortized bayesian meta-learning.
\newblock In \emph{International Conference on Learning Representations}, 2018.

\bibitem[Snell et~al.(2017)Snell, Swersky, and Zemel]{snell2017prototypical}
Jake Snell, Kevin Swersky, and Richard Zemel.
\newblock Prototypical networks for few-shot learning.
\newblock In \emph{Proceedings of the 31st International Conference on Neural
  Information Processing Systems}, pages 4080--4090, 2017.

\bibitem[Sung et~al.(2018)Sung, Yang, Zhang, Xiang, Torr, and
  Hospedales]{sung2018learning}
Flood Sung, Yongxin Yang, Li~Zhang, Tao Xiang, Philip~HS Torr, and Timothy~M
  Hospedales.
\newblock Learning to compare: Relation network for few-shot learning.
\newblock In \emph{Proceedings of the IEEE conference on computer vision and
  pattern recognition}, pages 1199--1208, 2018.

\bibitem[Ravi and Larochelle(2016)]{ravi2016optimization}
Sachin Ravi and Hugo Larochelle.
\newblock Optimization as a model for few-shot learning.
\newblock In \emph{International Conference on Learning Representations}, 2016.

\bibitem[Baxter(2000)]{baxter2000model}
Jonathan Baxter.
\newblock A model of inductive bias learning.
\newblock \emph{Journal of artificial intelligence research}, 12:\penalty0
  149--198, 2000.

\bibitem[Pentina and Lampert(2014)]{pentina2014pac}
Anastasia Pentina and Christoph Lampert.
\newblock A pac-bayesian bound for lifelong learning.
\newblock In \emph{International Conference on Machine Learning}, pages
  991--999, 2014.

\bibitem[Denevi et~al.(2019{\natexlab{a}})Denevi, Ciliberto, Grazzi, and
  Pontil]{denevi2019learning}
Giulia Denevi, Carlo Ciliberto, Riccardo Grazzi, and Massimiliano Pontil.
\newblock Learning-to-learn stochastic gradient descent with biased
  regularization.
\newblock In \emph{International Conference on Machine Learning}, pages
  1566--1575. PMLR, 2019{\natexlab{a}}.

\bibitem[Balcan et~al.(2019)Balcan, Khodak, and Talwalkar]{balcan2019provable}
Maria-Florina Balcan, Mikhail Khodak, and Ameet Talwalkar.
\newblock Provable guarantees for gradient-based meta-learning.
\newblock In \emph{International Conference on Machine Learning}, pages
  424--433. PMLR, 2019.

\bibitem[Xu and Raginsky(2017)]{xu2017information}
Aolin Xu and Maxim Raginsky.
\newblock Information-theoretic analysis of generalization capability of
  learning algorithms.
\newblock In \emph{Advances in Neural Information Processing Systems}, pages
  2524--2533, 2017.

\bibitem[Yoon et~al.(2018)Yoon, Kim, Dia, Kim, Bengio, and
  Ahn]{yoon2018bayesian}
Jaesik Yoon, Taesup Kim, Ousmane Dia, Sungwoong Kim, Yoshua Bengio, and Sungjin
  Ahn.
\newblock Bayesian model-agnostic meta-learning.
\newblock In \emph{Advances in Neural Information Processing Systems}, pages
  7332--7342, 2018.

\bibitem[Scaman and Virmaux(2018)]{scaman2018lipschitz}
Kevin Scaman and Aladin Virmaux.
\newblock Lipschitz regularity of deep neural networks: analysis and efficient
  estimation.
\newblock In \emph{Proceedings of the 32nd International Conference on Neural
  Information Processing Systems}, pages 3839--3848, 2018.

\bibitem[Li et~al.(2019)Li, Luo, and Qiao]{li2019generalization}
Jian Li, Xuanyuan Luo, and Mingda Qiao.
\newblock On generalization error bounds of noisy gradient methods for
  non-convex learning.
\newblock \emph{arXiv preprint arXiv:1902.00621}, 2019.

\bibitem[Welling and Teh(2011)]{welling2011bayesian}
Max Welling and Yee~W Teh.
\newblock Bayesian learning via stochastic gradient langevin dynamics.
\newblock In \emph{Proceedings of the 28th international conference on machine
  learning (ICML-11)}, pages 681--688, 2011.

\bibitem[Jose and Simeone(2020)]{jose2020information}
Sharu~Theresa Jose and Osvaldo Simeone.
\newblock Information-theoretic generalization bounds for meta-learning and
  applications.
\newblock \emph{arXiv preprint arXiv:2005.04372}, 2020.

\bibitem[Jiang et~al.(2019)Jiang, Kone{\v{c}}n{\`y}, Rush, and
  Kannan]{jiang2019improving}
Yihan Jiang, Jakub Kone{\v{c}}n{\`y}, Keith Rush, and Sreeram Kannan.
\newblock Improving federated learning personalization via model agnostic meta
  learning.
\newblock \emph{arXiv preprint arXiv:1909.12488}, 2019.

\bibitem[Liu et~al.(2019)Liu, Socher, and Xiong]{liu2019taming}
Hao Liu, Richard Socher, and Caiming Xiong.
\newblock Taming maml: Efficient unbiased meta-reinforcement learning.
\newblock In \emph{International Conference on Machine Learning}, pages
  4061--4071. PMLR, 2019.

\bibitem[Fallah et~al.(2020{\natexlab{a}})Fallah, Mokhtari, and
  Ozdaglar]{fallah2020personalized}
Alireza Fallah, Aryan Mokhtari, and Asuman Ozdaglar.
\newblock Personalized federated learning: A meta-learning approach.
\newblock \emph{arXiv preprint arXiv:2002.07948}, 2020{\natexlab{a}}.

\bibitem[Gui et~al.(2018)Gui, Wang, Ramanan, and Moura]{gui2018few}
Liang-Yan Gui, Yu-Xiong Wang, Deva Ramanan, and Jos{\'e}~MF Moura.
\newblock Few-shot human motion prediction via meta-learning.
\newblock In \emph{Proceedings of the European Conference on Computer Vision
  (ECCV)}, pages 432--450, 2018.

\bibitem[Nichol and Schulman(2018)]{nichol2018reptile}
Alex Nichol and John Schulman.
\newblock Reptile: a scalable metalearning algorithm.
\newblock \emph{arXiv preprint arXiv:1803.02999}, 2\penalty0 (3):\penalty0 4,
  2018.

\bibitem[Finn et~al.(2018)Finn, Xu, and Levine]{finn2018probabilistic}
Chelsea Finn, Kelvin Xu, and Sergey Levine.
\newblock Probabilistic model-agnostic meta-learning.
\newblock In \emph{Advances in Neural Information Processing Systems}, pages
  9516--9527, 2018.

\bibitem[Grant et~al.(2018)Grant, Finn, Levine, Darrell, and
  Griffiths]{grant2018recasting}
Erin Grant, Chelsea Finn, Sergey Levine, Trevor Darrell, and Thomas Griffiths.
\newblock Recasting gradient-based meta-learning as hierarchical bayes.
\newblock In \emph{International Conference on Learning Representations}, 2018.

\bibitem[Khodak et~al.(2019)Khodak, Balcan, and Talwalkar]{khodak2019adaptive}
Mikhail Khodak, Maria-Florina Balcan, and Ameet Talwalkar.
\newblock Adaptive gradient-based meta-learning methods.
\newblock \emph{arXiv preprint arXiv:1906.02717}, 2019.

\bibitem[Denevi et~al.(2019{\natexlab{b}})Denevi, Stamos, Ciliberto, and
  Pontil]{denevi2019online}
Giulia Denevi, Dimitris Stamos, Carlo Ciliberto, and Massimiliano Pontil.
\newblock Online-within-online meta-learning.
\newblock In \emph{ADVANCES IN NEURAL INFORMATION PROCESSING SYSTEMS 32 (NIPS
  2019)}, volume~32, pages 1--11. Neural Information Processing Systems
  (NeurIPS 2019), 2019{\natexlab{b}}.

\bibitem[Finn et~al.(2019)Finn, Rajeswaran, Kakade, and Levine]{finn2019online}
Chelsea Finn, Aravind Rajeswaran, Sham Kakade, and Sergey Levine.
\newblock Online meta-learning.
\newblock In \emph{International Conference on Machine Learning}, pages
  1920--1930. PMLR, 2019.

\bibitem[Fallah et~al.(2020{\natexlab{b}})Fallah, Mokhtari, and
  Ozdaglar]{fallah2020convergence}
Alireza Fallah, Aryan Mokhtari, and Asuman Ozdaglar.
\newblock On the convergence theory of gradient-based model-agnostic
  meta-learning algorithms.
\newblock In \emph{International Conference on Artificial Intelligence and
  Statistics}, pages 1082--1092. PMLR, 2020{\natexlab{b}}.

\bibitem[Ji et~al.(2020)Ji, Yang, and Liang]{ji2020multi}
Kaiyi Ji, Junjie Yang, and Yingbin Liang.
\newblock Multi-step model-agnostic meta-learning: Convergence and improved
  algorithms.
\newblock \emph{arXiv preprint arXiv:2002.07836}, 2020.

\bibitem[Denevi et~al.(2018)Denevi, Ciliberto, Stamos, and
  Pontil]{denevi2018learning}
Giulia Denevi, Carlo Ciliberto, Dimitris Stamos, and Massimiliano Pontil.
\newblock Learning to learn around a common mean.
\newblock In \emph{ADVANCES IN NEURAL INFORMATION PROCESSING SYSTEMS 31 (NIPS
  2018)}, volume~31. NIPS Proceedings, 2018.

\bibitem[Bai et~al.(2021)Bai, Chen, Zhou, Zhao, Lee, Kakade, Wang, and
  Xiong]{bai2021important}
Yu~Bai, Minshuo Chen, Pan Zhou, Tuo Zhao, Jason Lee, Sham Kakade, Huan Wang,
  and Caiming Xiong.
\newblock How important is the train-validation split in meta-learning?
\newblock In \emph{International Conference on Machine Learning}, pages
  543--553. PMLR, 2021.

\bibitem[Saunshi et~al.(2021)Saunshi, Gupta, and Hu]{saunshi2021representation}
Nikunj Saunshi, Arushi Gupta, and Wei Hu.
\newblock A representation learning perspective on the importance of
  train-validation splitting in meta-learning.
\newblock In \emph{International Conference on Machine Learning}, pages
  9333--9343. PMLR, 2021.

\bibitem[Russo and Zou(2019)]{russo2019much}
Daniel Russo and James Zou.
\newblock How much does your data exploration overfit? controlling bias via
  information usage.
\newblock \emph{IEEE Transactions on Information Theory}, 66\penalty0
  (1):\penalty0 302--323, 2019.

\bibitem[Bu et~al.(2020)Bu, Zou, and Veeravalli]{bu2020tightening}
Yuheng Bu, Shaofeng Zou, and Venugopal~V Veeravalli.
\newblock Tightening mutual information-based bounds on generalization error.
\newblock \emph{IEEE Journal on Selected Areas in Information Theory},
  1\penalty0 (1):\penalty0 121--130, 2020.

\bibitem[Pensia et~al.(2018)Pensia, Jog, and Loh]{pensia2018generalization}
Ankit Pensia, Varun Jog, and Po-Ling Loh.
\newblock Generalization error bounds for noisy, iterative algorithms.
\newblock In \emph{2018 IEEE International Symposium on Information Theory
  (ISIT)}, pages 546--550. IEEE, 2018.

\bibitem[Negrea et~al.(2019)Negrea, Haghifam, Dziugaite, Khisti, and
  Roy]{negrea2019information}
Jeffrey Negrea, Mahdi Haghifam, Gintare~Karolina Dziugaite, Ashish Khisti, and
  Daniel~M Roy.
\newblock Information-theoretic generalization bounds for sgld via
  data-dependent estimates.
\newblock In \emph{Advances in Neural Information Processing Systems}, pages
  11015--11025, 2019.

\bibitem[Steinke and Zakynthinou(2020)]{steinke2020reasoning}
Thomas Steinke and Lydia Zakynthinou.
\newblock Reasoning about generalization via conditional mutual information.
\newblock \emph{arXiv preprint arXiv:2001.09122}, 2020.

\bibitem[Maurer et~al.(2016)Maurer, Pontil, and
  Romera-Paredes]{maurer2016benefit}
Andreas Maurer, Massimiliano Pontil, and Bernardino Romera-Paredes.
\newblock The benefit of multitask representation learning.
\newblock \emph{The Journal of Machine Learning Research}, 17\penalty0
  (1):\penalty0 2853--2884, 2016.

\bibitem[McAllester and Stratos(2020)]{mcallester2020formal}
David McAllester and Karl Stratos.
\newblock Formal limitations on the measurement of mutual information.
\newblock In \emph{International Conference on Artificial Intelligence and
  Statistics}, pages 875--884. PMLR, 2020.

\bibitem[Raginsky et~al.(2017)Raginsky, Rakhlin, and
  Telgarsky]{raginsky2017non}
Maxim Raginsky, Alexander Rakhlin, and Matus Telgarsky.
\newblock Non-convex learning via stochastic gradient langevin dynamics: a
  nonasymptotic analysis.
\newblock In \emph{Conference on Learning Theory}, pages 1674--1703. PMLR,
  2017.

\bibitem[Lake et~al.(2011)Lake, Salakhutdinov, Gross, and
  Tenenbaum]{lake2011one}
Brenden Lake, Ruslan Salakhutdinov, Jason Gross, and Joshua Tenenbaum.
\newblock One shot learning of simple visual concepts.
\newblock In \emph{Proceedings of the annual meeting of the cognitive science
  society}, volume~33, 2011.

\bibitem[Vinyals et~al.(2016)Vinyals, Blundell, Lillicrap, Kavukcuoglu, and
  Wierstra]{vinyals2016matching}
Oriol Vinyals, Charles Blundell, Timothy Lillicrap, Koray Kavukcuoglu, and Daan
  Wierstra.
\newblock Matching networks for one shot learning.
\newblock \emph{arXiv preprint arXiv:1606.04080}, 2016.

\bibitem[Boucheron et~al.(2013)Boucheron, Lugosi, and
  Massart]{boucheron2013concentration}
St{\'e}phane Boucheron, G{\'a}bor Lugosi, and Pascal Massart.
\newblock \emph{Concentration inequalities: A nonasymptotic theory of
  independence}.
\newblock Oxford university press, 2013.

\bibitem[Germain et~al.(2016)Germain, Bach, Lacoste, and
  Lacoste-Julien]{germain2016pac}
Pascal Germain, Francis Bach, Alexandre Lacoste, and Simon Lacoste-Julien.
\newblock Pac-bayesian theory meets bayesian inference.
\newblock In \emph{Advances in Neural Information Processing Systems}, pages
  1884--1892, 2016.

\bibitem[Long(2018)]{MAML_Pytorch}
Liangqu Long.
\newblock Maml-pytorch implementation.
\newblock \url{https://github.com/dragen1860/MAML-Pytorch}, 2018.

\bibitem[Amit(2019)]{MALP}
Ron Amit.
\newblock meta-learning-adjusting-priors.
\newblock \url{https://github.com/ron-amit/meta-learning-adjusting-priors2},
  2019.

\end{thebibliography}
\newpage
\appendix

\section{Technical Lemmas} \label{Lemmas}
\begin{lemma}[Variational Form of Mutual Information]
Let $X$ and $Y$ be two random variables. For all probability measures $Q$ defined on the space of $X$, we have
\[
I(X;Y) \leq E_{Y}[D_{\text{KL}}(P_{X|Y}||Q)],
\]
with equality for $Q=P_X$.
\end{lemma}\label{Lemma_B.1}
\begin{proof}
\begin{equation*}
\begin{aligned}
    &I(X;Y) + D_{\text{KL}}(P_X||Q) \\
    &= \iint p(x,y)\log\frac{p(x,y)}{p(x)p(y)}dxdy
    + \int p(x)\log\frac{p(x)}{q(x)}dx\\
    &= \iint p(x,y)\log\frac{p(x,y)}{p(x)p(y)}dxdy + \iint p(x,y)\log\frac{p(x)}{q(x)}dxdy\\
    &= \iint p(x,y)\log\frac{p(x|y)}{q(x)}dxdy\\
    &= \mathbb{E}_{Y}[D_{\text{KL}}(P(X|Y)||Q)]\, .
\end{aligned}
\end{equation*}       
Since $D_{\text{KL}}(P_X||Q) \geq 0$, the equality exists only when $Q=P_X$, which concludes the proof.
\end{proof}

\begin{lemma} Let $X, Y, Z$ be random variables. For all $\mathcal{Z}$-measurable probability measures $Q$ on the space of $X$, $I^Z(X;Y) \leq E_{Y|Z}[D_{\text{KL}}(P_{X|Y,Z}||Q)]$, with equality for $Q=P_{X|Z}$.
\end{lemma}\label{Lemma_B.2}
\begin{proof}
\begin{equation*}
\begin{aligned}
    &I^Z(X;Y) + D_{\text{KL}}(P_{X|Z}||Q)\\
    &= \iint p(x,y|z)\log\frac{p(x,y|z)}{p(x|z)p(y|z)}dxdy
    + \iint p(x|z)\log\frac{p(x|z)}{q(x)}dx\\
    &= \iint p(x,y|z)\log\frac{p(x,y|z)}{p(x|z)p(y|z)}dxdy\\
    &\quad\quad\quad\quad + \iint p(x,y|z)\log\frac{p(x|z)}{q(x)}dxdy\\
    &= \iint p(x,y|z)\log\frac{p(x|y,z)}{q(x)}dxdy\\
    &= \mathbb{E}_{Y|Z}[D_{\text{KL}}(P(X|Y,Z)||Q)]
\end{aligned}
\end{equation*}
Since $D_{\text{KL}}(P_{X|Z}||Q) \geq 0$, the equality exists only when $Q=P_{X|Z}$, which concludes the proof.
\end{proof}

\begin{lemma} Let $X, Y, Z$ be random variables. For all $\mathcal{Z}$-measurable probability measures $Q$ defined on the space of $X$, $I(X;Y|Z) = \mathbb{E}_{Z}[I^{Z}(X;Y)] \leq E_{Y,Z}[D_{\text{KL}}(P_{X|Y,Z}||Q)]$, with equality for $Q=P_{X|Z}$.
\end{lemma}\label{Lemma_B.3}
\begin{proof}
Take the expectation on the inequality of Lemma B.2 to obtain the result.
\end{proof}

\begin{lemma}\textbf{(Donsker-Varadhan representation[Corollary 4.15\cite{boucheron2013concentration}])} Let $P$ and $Q$ be two probability measures defined on a set $\mathcal{X}$. Let $g : \mathcal{X} \rightarrow R$ be a measurable function, and let $\mathbb{E}_{x \sim Q}[\exp{g(x)}] \leq \infty$. Then \[D_{\text{KL}}(P||Q) = \sup\limits_{g}\{\mathbb{E}_{x\sim P}[g(x)] - \log\mathbb{E}_{x \sim Q}[\exp{g(x)}]\}.\]
\end{lemma}\label{Lemma_B.4}
\begin{lemma}
\textbf{(Decoupling Estimate[\citet{xu2017information}])} 
Consider a pair of random variables $X$ and $Y$ with joint distribution $P_{X,Y}$, let $\tilde{X}$ be an independent copy of $X$, and $\tilde{Y}$ an independent copy of $Y$, such that $P_{\tilde{X},\tilde{Y}} = P_{X}P_{Y}$. For arbitrary real-valued function $f:\mathcal{X}\times\mathcal{Y}\rightarrow \Reals$, if $f(\tilde{X}, \tilde{Y})$ is $\sigma$-subgaussian under $P_{\tilde{X},\tilde{Y}}$, then:
\[
|\mathbb{E}[f(X,Y)] - \mathbb{E}[f(\tilde{X},\tilde{Y})]| \leq \sqrt{2\sigma^2I(X;Y)}
\]
\end{lemma}\label{Lemma_B.5}

\begin{lemma}
Let Q be an arbitrary distribution on $\mathcal{W}$, and let $S$ be an arbitrary sample of examples. The solution to the optimization problem
\begin{equation*}
P^*\ =\ \arg\inf\limits_{P}\left\{ \mathbb{E}_{W\sim P}[R_S(W)]\ +\  \frac{1}{\beta}D_{\text{KL}}(P||Q) \right\}\, . 
\end{equation*}
is given by the Gibbs distribution 
\[
dP^*(w)\ =\ \frac{e^{-\beta R_S(w)}dQ(w)}{\mathbb{E}_{W\sim Q}e^{-\beta R_S(W)}}\, .
\]
\end{lemma}\label{Lemma_B.6}

\begin{lemma}
\textbf{(Data Processing Inequality)}
Given random variables $X,Y,Z,V$, and the Markov Chain: \[X\rightarrow Y\rightarrow Z,\] then we have
\[I(X;Z) \leq I(X;Y)\,,I(X;Z)\leq I(Y;Z).\]
For Markov chain 
\[V\rightarrow X\rightarrow Y\rightarrow Z\,,\]
we have
\[I(X;Z|V) \leq I(X;Y|V), I(X;Z|V) \leq I(Y;Z|V)\]
\end{lemma}\label{Lemma_B.7}
\begin{proof} 
Since $$I(X;Y,Z) = I(X;Z) + I(X;Y|Z) = I(X;Y) + I(X;Z|Y)\,,$$ and with the Markov Chain, we have $X \indep Z | Y$, therefore 
\[I(X;Z|Y) = H(X|Y) -H(X|Y,Z) = 0\,.\]
In addition, $I(X;Y|Z) \geq 0$, so $I(X;Z) \leq I(X;Y)$. 
\[I(Z;X,Y) = I(Z;X) + I(Z;Y|X) = I(Z;Y) + I(Z;X|Y) = I(Y:Z)\,,\]
with $I(Y;Z|X) \geq 0$, we have $I(X;Z) \leq I(Y;Z)$.

Similarly, for the second Markov chain, we have $X \indep Z | Y,V$, therefore
\[I(X;Z|Y,V) = H(X|Y,V) -H(X|Y,Z,V) = 0\,.\]
\[I(X;Y,Z|V) = I(X;Z|V) + I(X;Y|V,Z) = I(X;Y|V) + I(X;Z|Y,V)=I(X;Y|V)\]
So we have $I(X;Z|V) \leq I(X;Y|V)$, the rest proof is similar and omitted.
\end{proof}

\begin{lemma}
Given random variables $X,Y,Z_1,Z_2$, and the graph model: 
\[Z_1\rightarrow Z_2\rightarrow X \leftarrow Y\,,\]
then we have 
\[I(X;Y|Z_1) \leq I(X;Y|Z_2)\]
\end{lemma}\label{Lemma_B.8}
\begin{proof}
Apply chain rule, we get:
\[I(X;Y,Z_2|Z_1) = I(X;Y|Z_1) + I(X;Z_2|Y,Z_1) = I(X;Z_2|Z_1) + I(X;Y|Z_2,Z_1)\]
From the graph model, we have $Y \indep Z_1$, $Y \indep Z_2$ and $(X,Y) \indep Z_1 | Z_2$.
Hence
\[I(X;Y|Z_2,Z_1) = H(X|Z_2,Z_1) - H(X|Y,Z_2,Z_1)= H(X|Z_2) - H(X|Y,Z_2) = I(X;Y|Z_2)\]
Moreover,
\[\begin{aligned}
I(X,Y;Z_2|Z_1) &= I(X;Z_2|Z_1) + I(Y;Z_2|X,Z_1) \\
&=I(Y;Z_2|Z_1) + I(Z_2;X|Y,Z_1)\\
&=I(Z_2;X|Y,Z_1)
\end{aligned}\]
the last equality is obtained with $Y \indep Z_2$ and $Y \indep Z_1$, since $I(Y;Z_2|X,Z_1) \geq 0$, we get
$I(X;Z_2|Z_1) \leq I(X;Z_2|Y,Z_1)$.
Consequently, we have $I(X;Y|Z_1) \leq I(X;Y|Z_2)$, conclude the proof.
\end{proof}

\section{Proof}
\subsection{Proof of Theorem 5.1} \label{proof_thm_mi}

\begin{theorem*}[Meta-generalization error bound for joint training]

Suppose all tasks use the same loss $\ell(Z,w)$, which is $\sigma$-subgaussian for any $w \in \mathcal{W}$, where $Z \sim \mu, \mu \sim \tau$.Then, the meta generalization error for joint training is upper bounded by 
\begin{equation*}
 |\text{gen}^{\text{joi}}_{\text{meta}}(\tau,\mathcal{A}_{\text{meta}}, \mathcal{A}_{\text{base}})| \leq  \sqrt{\frac{2\sigma^2}{nm} I(U,W_{1:n};S_{1:n})}\, .
\end{equation*}
\end{theorem*}

\begin{proof}
In contrast to previous works \cite{pentina2014pac,amit2018meta,jose2020information}, which separately bound the environment-level and task-level error and then combine the two terms, we consider $U,W_{1:n}$ as a collection and directly bound the whole term. By using the chain rule for mutual information, the final result can then be split into an environment-level and a task-level contribution. 

Similar to Lemma 2.5, let $\Phi = (U,W_{1:n})\in \mathcal{U}\times\mathcal{W}^n$ be a collection of random variables such that $\Phi \not\indep  S_{1:n}$, and let
$\tilde{\Phi} = (\tilde{U},\tilde{W}_{1:n})\in \mathcal{U}\times\mathcal{W}^n$ be an in dependant copy of $\Phi$ such that $\tilde{\Phi} \indep  S_{1:n}$, \ie, $\tilde{\Phi}$ is distributed according to $P_{U,W_{1:n}} = \mathbb{E}_{S_{1:n}} P_{U,W_{1:n}|S_{1:n}}$. Let 

\[f(\Phi, S_{1:n})\eqdef \frac{1}{n}\sum_{i=1}^{n} [R_{S_i}(W_i)] = \frac{1}{n}\sum_{i=1}^{n} \frac{1}{m}\sum_{j=1}^m \ell(W_i, Z_{i,j}).\]

For any $\lambda\in\Reals$, let 
\begin{equation*}
    \begin{aligned}
    \psi_{\tilde{\Phi}, S_{1:n}}(\lambda) &\eqdef\log\mathbb{E}_{\tilde{\Phi}, S_{1:n}}\left[e^{\lambda(f(\tilde{\Phi}, S_{1:n}) - \mathbb{E}[f(\tilde{\Phi}, S_{1:n})]}\right] \\
    &= \log\mathbb{E}_{\tilde{\Phi}, S_{1:n}}[e^{\lambda f(\tilde{\Phi}, S_{1:n})}] - \lambda\mathbb{E}_{\tilde{\Phi}, S_{1:n}}[f(\tilde{\Phi}, S_{1:n})]\, .
    \end{aligned}
\end{equation*}
Moreover,
\begin{align}
\nonumber I(\Phi;S_{1:n}) &= D_{\text{KL}}(P_{\Phi,S_{1:n}}||P_\Phi P_{S_{1:n}})\\
\nonumber &= \sup\limits_{g}\left\{\mathbb{E}_{\Phi,S_{1:n}}[g(\Phi, S_{1:n})] - \log\mathbb{E}_{\tilde{\Phi},S_{1:n}}[e^{g(\tilde{\Phi},S_{1:n})}]\right\}\\
\nonumber &\geq \lambda\mathbb{E}_{\Phi,S_{1:n}}[f(\Phi, S_{1:n})] - \log\mathbb{E}_{\tilde{\Phi},S_{1:n}}[e^{\lambda f(\tilde{\Phi},S_{1:n})}]\, ,\quad \forall \lambda\in\Reals\\
\nonumber &=\lambda\mathbb{E}_{U, W_{1:n},S_{1:n}}[f(\Phi, S_{1:n})] - \lambda\mathbb{E}_{\tilde{U}, \tilde{W}_{1:n},S_{1:n}}[f(\tilde{\Phi}, S_{1:n})] - \psi_{\tilde{\Phi}, S_{1:n}}(\lambda)\\
\label{eq:dv} &=\lambda\mathbb{E}_{U, W_{1:n},S_{1:n}}\frac{1}{n}\sum_{i=1}^{n}[R_{S_i}(W_i)] 
- \lambda\mathbb{E}_{\tilde{U}, \tilde{W}_{1:n},S_{1:n}}\frac{1}{n}\sum_{i=1}^{n} [R_{S_i}(\tilde{W}_i)] - \psi_{\tilde{\Phi}, S_{1:n}}(\lambda)
\end{align}

Since $(W_i, S_i), i=1,...,n$ are mutually independent given $U$, and $S_1,...S_n$ are independent, we have $p(w_{1:n}|s_{1:n},u) = \prod_{i=1}^n p(w_i|s_i,u)$. Hence 
\begin{align}
    \nonumber \lambda\mathbb{E}_{U, W_{1:n},S_{1:n}}\frac{1}{n}\sum_{i=1}^{n}[R_{S_i}(W_i)] 
    \nonumber &= \lambda\mathbb{E}_{U,S_{1:n}}\frac{1}{n}\sum_{i=1}^{n}\mathbb{E}_{W_i|S_i,U}[R_{S_i}(W_i)] \\
    \label{eq:dv_1}&= \lambda\mathbb{E}_{U, S_{1:n}}[R_{S_{1:n}}(U)]
\end{align}
Since $\tilde{\Phi} \indep  S_{1:n}$, we have that $P_{\tilde{W}_{1:n}|S_{1:n},\tilde{U}}=P_{\tilde{W}_{1:n}|\tilde{U}}$ and $P_{\tilde{W}_{1:n},S_{1:n},\tilde{U}} = P_{\tilde{W}_{1:n},\tilde{U}}P_{S_{1:n}}$. Hence,
\begin{align*}
R_{\tau}(\tilde{U}) 
&= \mathbb{E}_{S \sim \mu_{m,\tau}} \mathbb{E}_{\tilde{W} \sim P_{\tilde{W}|S, \tilde{U}}}[R_{\mu}(\tilde{W})] 
= \mathbb{E}_{\mu \sim \tau} \mathbb{E}_{S|\mu \sim \mu^m} \mathbb{E}_{\tilde{W} \sim P_{\tilde{W}|\tilde{U}}}[R_{\mu}(\tilde{W})] \\
&= \mathbb{E}_{\mu \sim \tau} \mathbb{E}_{\tilde{W} \sim P_{\tilde{W}|\tilde{U}}}[R_{\mu}(\tilde{W})]\, . 
\end{align*}
Therefore
\begin{align}
\nonumber \lambda\mathbb{E}_{\tilde{U}, \tilde{W}_{1:n},S_{1:n}}\frac{1}{n}\sum_{i=1}^{n} R_{S_i}(\tilde{W}_i) 
\nonumber &= \lambda\mathbb{E}_{\tilde{U},\tilde{W}_{1:n}} \mathbb{E}_{S_{1:n}\sim \mu_{m,\tau}^n} \left[\frac{1}{n}\sum_{i=1}^{n}R_{S_i}(\tilde{W}_i)\right]\\
\nonumber &= \lambda\mathbb{E}_{\tilde{U},\tilde{W}_{1:n}} \left[\frac{1}{n}\sum_{i=1}^{n}\mathbb{E}_{S_i \sim \mu_{m,\tau}}R_{S_i}(\tilde{W}_i)\right]\\
\nonumber &= \lambda\mathbb{E}_{\tilde{U},\tilde{W}_{1:n}}  \left[\frac{1}{n}\sum_{i=1}^{n}\mathbb{E}_{\mu_i\sim \tau}
\mathbb{E}_{S_i|\mu_i\sim \mu_{i}^m} R_{S_i}(\tilde{W}_i)\right]\\
\nonumber &= \lambda\mathbb{E}_{\tilde{U},\tilde{W}_{1:n}} \left[\frac{1}{n}\sum_{i=1}^{n} \mathbb{E}_{\mu_i \sim \tau}\left[\frac{1}{m}\sum_{j=1}^m \mathbb{E}_{Z_{i,j} \sim \mu_i} \ell(\tilde{W}_i, Z_{i,j})\right]\right]\\
\nonumber &= \lambda\mathbb{E}_{\tilde{U}}\mathbb{E}_{\tilde{W}_{1:n}|\tilde{U}}\left[\frac{1}{n}\sum_{i=1}^{n} \mathbb{E}_{\mu_i \sim \tau}\left[\frac{1}{m}\sum_{j=1}^m \mathbb{E}_{Z_{i,j} \sim \mu_i} \ell(\tilde{W}_i, Z_{i,j})\right]\right]\\
\nonumber &= \lambda\mathbb{E}_{\tilde{U}}\left[\frac{1}{n}\sum_{i=1}^{n} \mathbb{E}_{\tilde{W}_{i}|\tilde{U}}\mathbb{E}_{\mu_i \sim \tau} R_{\mu_i}(\tilde{W}_i)\right]\\
\nonumber &= \lambda\mathbb{E}_{\tilde{U}} \mathbb{E}_{\mu\sim \tau} \mathbb{E}_{\tilde{W}|\tilde{U}}R_{\mu}(\tilde{W})
= \lambda\mathbb{E}_{\tilde{U}} R_\tau(\tilde{U})\\
\label{eq:dv_2}&= \lambda\mathbb{E}_{U,S_{1:n}}R_\tau(U)\, .
\end{align}
If we use Equations~(\ref{eq:dv_1}) and (\ref{eq:dv_2}), then Equation (\ref{eq:dv}) becomes
\begin{equation}\label{eq:dv_3}
-\lambda\EE_{U,S_{1:n}}\left[ R_\tau(U) - R_{S_{1:n}}(U) \right]\ \le\ I(\Phi;S_{1:n}) + \psi_{\tilde{\Phi},S_{1:n}}(\lambda)\, ,\quad\forall\lambda\in\Reals .
\end{equation}
Since this inequality is also valid when $\lambda$ is negative, this implies that we also have
\[
\EE_{U,S_{1:n}}\left[ R_\tau(U) - R_{S_{1:n}}(U) \right] \le \frac{1}{\lambda} \left[ I(\Phi;S_{1:n}) + \psi_{\tilde{\Phi},S_{1:n}}(-\lambda)\right]\, ,\quad\forall\lambda > 0\, .
\]
Consequently,
\[
\text{gen}_{\text{meta}}^{\text{joi}}(\tau,\mathcal{A}_{\text{meta}}, \mathcal{A}_{\text{base}})\ \le\ \frac{1}{\lambda} \left[ I(\Phi;S_{1:n}) + \psi_{\tilde{\Phi},S_{1:n}}(-\lambda)\right]\, ,\quad\forall\lambda > 0\, .
\]
Since $\ell(\tilde{W}, Z)$ is $\sigma$-subgaussian, we have that $f(\tilde{\Phi}, S_{1:n}) =  \frac{1}{n}\sum_{i=1}^{n} \frac{1}{m}\sum_{j=1}^m \ell(\tilde{W}_i, Z_{i,j})$ is $\frac{\sigma}{\sqrt{nm}}$-subgaussian.\footnote{More discussion on subgaussianity can be found in Section \ref{sub}.} Hence, 
\[
\psi_{\tilde{\Phi},S_{1:n}}(\lambda)\ \le\ \frac{\lambda^2\sigma^2}{2nm}\quad \forall\lambda\in\Reals\, .
\]
Thus, we have
\[
\text{gen}_{\text{meta}}^{\text{joi}}(\tau,\mathcal{A}_{\text{meta}}, \mathcal{A}_{\text{base}})\ \le\ \frac{I(\Phi;S_{1:n})}{\lambda}  + \frac{\lambda\sigma^2}{2nm}\, ,\quad\forall\lambda > 0\, .
\]
By using the value of $\lambda$ that minimizes the r.h.s.\ of the above equation, we have
\begin{equation}\label{eq:ub_joi}
     \text{gen}_{\text{meta}}^{\text{joi}}(\tau,\mathcal{A}_{\text{meta}}, \mathcal{A}_{\text{base}}) \leq \sqrt{\frac{2\sigma^2I(\Phi;S_{1:n})}{nm}}\, .
\end{equation}
Returning to Equation~(\ref{eq:dv_3}), we have for $\lambda > 0$:
$$
\EE_{U,S_{1:n}}\left[ R_\tau(U) - R_{S_{1:n}}(U) \right] \ge -\frac{1}{\lambda} \left[ I(\Phi;S_{1:n}) + \psi_{\tilde{\Phi},S_{1:n}}(\lambda)\right] \ge - \sqrt{\frac{2\sigma^2I(\Phi;S_{1:n})}{nm}}\, .
$$
Hence, we also have
\begin{equation}\label{eq:lb_joi}
    \text{gen}_{\text{meta}}^{\text{joi}}(\tau,\mathcal{A}_{\text{meta}}, \mathcal{A}_{\text{base}}) \geq -\sqrt{\frac{2\sigma^2I(\Phi;S_{1:n})}{nm}}\, .
\end{equation}
Then, Equations~(\ref{eq:ub_joi}) and~(\ref{eq:lb_joi}) together imply that
$$
\left| \text{gen}_{\text{meta}}^{\text{joi}}(\tau,\mathcal{A}_{\text{meta}}, \mathcal{A}_{\text{base}})\right| \leq \sqrt{\frac{2\sigma^2I(\Phi;S_{1:n})}{nm}}\, ,
$$
which gives the theorem.
\end{proof}


\subsection{Benefits of Meta Learning}\label{Proof_ben}

The task specific empirical risk $R_S(W)$ is independent of the meta parameter $U$, given the task specific parameter $W$, which gives the implicit independence assumption $S \indep U|W$. We thus have $I(U;S|W)=0$, and the following two possible decompositions:
$$I(S;U,W) = I(W;S|U) + I(U;S) = I(W;S) + I(U;S|W) = I(W;S)\,.$$
Since $I(U;S) \geq 0$, we obtain $I(W;S|U) \leq I(W;S)$.

As mentioned in the main paper, Theorem 5.1 can cover the PAC Bayes bound of \citet{amit2018meta} with the variational form of mutual information. Their work has built a connection between PAC Bayes meta-learning and Hierarchical Variational Bayes. 
In Appendix A.3 of \cite{amit2018meta}, they give the generative graph model for meta learning where $U \rightarrow W \rightarrow S$ (their notation used $\psi$ instead of $U$). They assumed that $S$ is independent of $U$ given $W$, in Bayes learning, this implies that $p(S|W,U)=p(S|W)$. Based on the graph model, they obtained a similar optimization objective as their PAC-Bayes meta learning algorithm, which minimizes the expected empirical risk plus the PAC Bayes bound. \citet{germain2016pac} has given a more obvious connection between PAC Bayes learning and Bayes learning, where optimizing the PAC Bayes bound together with the expected empirical risk gives the so called Gibbs algorithm (see Lemma~\ref{Lemma_B.6}.6). When using the negative log loss, \ie, $R_S(W) = -\frac{1}{m}\log p(S|W) = - \frac{1}{m} \sum_{i=1}^m \log p(Z_i|W)$, the output of Gibbs algorithm coincides with the Bayes Posterior. Therefore, without the independence assumption, $R_S(W)$ should be defined as $R_S(W,U)$, which corresponds to $-\frac{1}{m}\log p(S|W,U)$ in Bayes learning.

\subsection{Proof of Theorem 5.2}\label{proof_thm_cmi}
\begin{theorem*}[Meta-generalization error bound for alternate training]
Suppose all tasks use the same loss $\ell(Z,w)$, which is $\sigma$-subgaussian for any $w \in \mathcal{W}$, where $Z \sim \mu, \mu \sim \tau$. Then we have
\[
 |\text{gen}^{\text{alt}}_{\text{meta}}(\tau, \mathcal{A}_{\text{meta}}, \mathcal{A}_{\text{base}})| \leq \mathbb{E}_{S^{\text{tr}}_{1:n}}\sqrt{\frac{2\sigma^2I^{S^{\text{tr}}_{1:n}}(U, W_{1:n};S^{\text{va}}_{1:n})}{nm_{\text{va}}}}
\leq\sqrt{\frac{2\sigma^2I(U,W_{1:n};S^{\text{va}}_{1:n}|S^{\text{tr}}_{1:n})}{nm_{\text{va}}}}\, .
\]
\end{theorem*}

\begin{proof}
The proof technique is analogous to Theorem 5.1. 
Let $\Phi = (U,W_{1:n})$ be a collection of random variables where $\Phi\in \mathcal{U}\times\mathcal{W}^n$ such that $\Phi$ and $S_{1:n}$ follow the joint distribution $P_{\Phi,S_{1;n}}$. Then let $\tilde{\Phi}$ be an independent copy of $\Phi$, such that $\tilde{\Phi} \indep  \{S^{\text{va}}_{1:n},S_{1:n}^{\text{tr}}\}$, \ie, $\tilde{\Phi}\sim \mathbb{E}_{S_{1:n}}P_{\Phi|S_{1:n}}$.
Define 
\[f(\Phi, S^{\text{va}}_{1:n})= \frac{1}{n}\sum_{i=1}^{n} [R_{S^{\text{va}}_i}(W_i)] = \frac{1}{n}\sum_{i=1}^{n} \frac{1}{m}\sum_{j=1}^{m_{\text{va}}} \ell(W_i, Z_{i,j})\,.\]
For any $\lambda \in \Reals$, denote the cumulant generation function of $\tilde{\Phi},  S^{\text{va}}_{1:n}|S^{\text{tr}}_{1:n}$ as:
\begin{equation*}
    \begin{aligned}
    \psi_{\tilde{\Phi},  S^{\text{va}}_{1:n}|S^{\text{tr}}_{1:n}}(\lambda) 
    &=\log\mathbb{E}_{\tilde{\Phi}, S^{\text{va}}_{1:n}|S^{\text{tr}}_{1:n}}[e^{\lambda(f(\tilde{\Phi}, S^{\text{va}}_{1:n}) - \mathbb{E}_{\tilde{\Phi},  S^{\text{va}}_{1:n}|S^{\text{tr}}_{1:n}}[f(\tilde{\Phi}, S^{\text{va}}_{1:n})]}] \\
    &= \log\mathbb{E}_{\tilde{\Phi},  S^{\text{va}}_{1:n}|S^{\text{tr}}_{1:n}}[e^{\lambda f(\tilde{\Phi}, S^{\text{va}}_{1:n})}] - \lambda \mathbb{E}_{\tilde{\Phi},  S^{\text{va}}_{1:n}|S^{\text{tr}}_{1:n}}[f(\tilde{\Phi}, S^{\text{va}}_{1:n})]
    \end{aligned}
\end{equation*}

In addition, the disintegrated mutual information is given as:

\begin{align}
\nonumber I^{S^{\text{tr}}_{1:n}}(\Phi;S^{\text{va}}_{1:n}) &= D_{\text{KL}}(P_{\Phi,S^{\text{va}}_{1:n}|S^{\text{tr}}_{1:n}}||P_{\tilde{\Phi}|S^{\text{tr}}_{1:n}}P_{S^{\text{va}}_{1:n}|S^{\text{tr}}_{1:n}})\\
\nonumber&= \sup\limits_{g}\left\{\mathbb{E}_{\Phi, S^{\text{va}}_{1:n}|S^{\text{tr}}_{1:n}}[g(\Phi,S^{\text{va}}_{1:n})]
- \log\mathbb{E}_{\tilde{\Phi}, S^{\text{va}}_{1:n}|S^{\text{tr}}_{1:n}}\left[e^{g(\tilde{\Phi},S^{\text{va}}_{1:n})}\right]\right\}\\
\nonumber&\geq \lambda\mathbb{E}_{\Phi,S^{\text{va}}_{1:n}|S^{\text{tr}}_{1:n}}[f(\Phi, S^{\text{va}}_{1:n})] 
- \log\mathbb{E}_{\tilde{\Phi},S^{\text{va}}_{1:n}|S^{\text{tr}}_{1:n}}\left[e^{\lambda f(\tilde{\Phi},S^{\text{va}}_{1:n})}\right]\\
\nonumber&=\lambda\mathbb{E}_{U, W_{1:n},S^{\text{va}}_{1:n}|S^{\text{tr}}_{1:n}}[f(\Phi, S^{\text{va}}_{1:n})]\\
\nonumber&\quad\quad\quad\quad\quad\quad- \lambda\mathbb{E}_{\tilde{U}, \tilde{W}_{1:n},S^{\text{va}}_{1:n}|S^{\text{tr}}_{1:n}}[f(\tilde{\Phi}, S^{\text{va}}_{1:n})] - \psi_{\tilde{\Phi}, S^{\text{va}}_{1:n}|S^{\text{tr}}_{1:n}}(\lambda)\\
\nonumber&=\lambda\mathbb{E}_{U, W_{1:n},S^{\text{va}}_{1:n}|S^{\text{tr}}_{1:n}}\frac{1}{n}\sum_{i=1}^{n} R_{S^{\text{va}}_i}(W_i)\\
\label{eq:dmi}&\quad\quad\quad\quad\quad\quad- \lambda\mathbb{E}_{\tilde{U}, \tilde{W}_{1:n},S^{\text{va}}_{1:n}|S^{\text{tr}}_{1:n}}\frac{1}{n}\sum_{i=1}^{n} R_{S^{\text{va}}_i}(\tilde{W}_i)
- \psi_{\tilde{\Phi}, S^{\text{va}}_{1:n}|S^{\text{tr}}_{1:n}}(\lambda)
\end{align}

 Since given $U$, $(W_i, S^{\text{tr}}_i), i=1,...,n$ are mutually independent, we have $p(w_{1:n}|s^{\text{tr}}_{1:n},u) = \prod_{i=1}^n p(w_i|s^{\text{tr}}_i,u)$. Thus 

\begin{align}
\nonumber\lambda\mathbb{E}_{U, W_{1:n},S^{\text{va}}_{1:n}|S^{\text{tr}}_{1:n}}\frac{1}{n}\sum_{i=1}^{n}[R_{S^{\text{va}}_i}(W_i)]&= \lambda\mathbb{E}_{U,S^{\text{va}}_{1:n}|S^{\text{tr}}_{1:n}}\frac{1}{n}\sum_{i=1}^{n}\mathbb{E}_{W_i|S^{\text{tr}}_i,U}[R_{S^{\text{va}}_i}(W_i)]\\
\label{eq:em_risk}&= \lambda\mathbb{E}_{U, S^{\text{va}}_{1:n}|S^{\text{tr}}_{1:n}}[\tilde{R}_{S_{1:n}}(U)]
\end{align}

Since we have $\tilde{\Phi} \indep  \{S^{\text{va}}_{1:n},S_{1:n}^{\text{tr}}\}$, 
thus $P_{\tilde{W}_{1:n}|S_{1:n},\tilde{U}}=P_{\tilde{W}_{1:n}|\tilde{U}}$, we have:

\begin{align}
R_{\tau}(\tilde{U}) 
\nonumber&= \mathbb{E}_{S \sim \mu_{m,\tau}} \mathbb{E}_{\tilde{W} \sim P_{\tilde{W}|S, \tilde{U}}}[R_{\mu}(\tilde{W})]
\nonumber= \mathbb{E}_{\mu \sim \tau} \mathbb{E}_{S|\mu \sim \mu^m} \mathbb{E}_{\tilde{W} \sim P_{\tilde{W}|\tilde{U}}}[R_{\mu}(\tilde{W})]\\
\nonumber&= \mathbb{E}_{\mu \sim \tau} \mathbb{E}_{\tilde{W} \sim P_{\tilde{W}|\tilde{U}}}[R_{\mu}(\tilde{W})] 
\end{align}

Moreover, we have $S_{1:n}^{\text{tr}}  \indep S_{1:n}^{\text{va}}$, so that $P_{\tilde{W}_{1:n},S^{\text{va}}_{1:n},\tilde{U}|S_{1:n}^{\text{tr}}} = P_{\tilde{W}_{1:n},\tilde{U}}P_{S^{\text{va}}_{1:n}|S_{1:n}^{\text{tr}}} = P_{\tilde{W}_{1:n},\tilde{U}}P_{S^{\text{va}}_{1:n}}$. Then we can also prove:

\begin{align}
\nonumber\lambda\mathbb{E}_{\tilde{U}, \tilde{W}_{1:n},S^{\text{va}}_{1:n}|S^{\text{tr}}_{1:n}}&\left[\frac{1}{n}\sum_{i=1}^{n} R_{S^{\text{va}}_i}(\tilde{W}_i)\right]
\nonumber= \lambda\mathbb{E}_{\tilde{U},\tilde{W}_{1:n}|S^{\text{tr}}_{1:n}} \mathbb{E}_{S^{\text{va}}_{1:n}\sim \mu_{m_{\text{va}},\tau}^n} \left[\frac{1}{n}\sum_{i=1}^{n}R_{S_i^{\text{va}}}(\tilde{W}_i)\right]\\
\nonumber&= \lambda\mathbb{E}_{\tilde{U},\tilde{W}_{1:n}|S^{\text{tr}}_{1:n}} \left[\frac{1}{n}\sum_{i=1}^{n}\mathbb{E}_{S_i \sim \mu_{m_{\text{va}},\tau}}R_{S^{\text{va}}_i}(\tilde{W}_i)\right]\\
\nonumber&= \lambda\mathbb{E}_{\tilde{U},\tilde{W}_{1:n}|S^{\text{tr}}_{1:n}}  \left[\frac{1}{n}\sum_{i=1}^{n}\mathbb{E}_{\mu_i\sim \tau}
\mathbb{E}_{S_i|\mu_i\sim \mu_{i}^{m_{\text{va}}}} [R_{S^{\text{va}}_i}(\tilde{W}_i)]\right]\\
\nonumber&= \lambda\mathbb{E}_{\tilde{U},\tilde{W}_{1:n}|S^{\text{tr}}_{1:n}} \left[\frac{1}{n}\sum_{i=1}^{n} \mathbb{E}_{\mu_i \sim \tau}[\frac{1}{m_{\text{va}}}\sum_{j=1}^{m_{\text{va}}} \mathbb{E}_{Z_{i,j} \sim \mu_i} \ell(\tilde{W}_i, Z_{i,j})]\right]\\
\nonumber&= \lambda\mathbb{E}_{\tilde{U}|S_{1:n}^{\text{tr}}}\left[\frac{1}{n}\sum_{i=1}^{n} \mathbb{E}_{\tilde{W}_{i}|\tilde{U}}\mathbb{E}_{\mu_i \sim \tau}[\frac{1}{m_{\text{va}}}\sum_{j=1}^{m_{\text{va}}} \mathbb{E}_{Z_{i,j} \sim \mu_i} \ell(\tilde{W}_i, Z_{i,j})]\right]\\
\nonumber&= \lambda\mathbb{E}_{\tilde{U}|S_{1:n}^{\text{tr}}} \left[\mathbb{E}_{\mu\sim \tau} \mathbb{E}_{\tilde{W}|\tilde{U}}[R_{\mu}(\tilde{W})]\right]\\
&= \lambda\mathbb{E}_{\tilde{U}|S_{1:n}^{\text{tr}}} R_\tau(\tilde{U})
\label{eq:true_risk}= \lambda\mathbb{E}_{U,S^{\text{va}}_{1:n}|S^{\text{tr}}_{1:n}} R_\tau(U)
\end{align}

Therefore, by combining Equations (\ref{eq:dmi}), (\ref{eq:em_risk}), and (\ref{eq:true_risk}), we have for any $\lambda$, 
\[
\lambda\mathbb{E}_{U, S^{\text{va}}_{1:n}|S^{\text{tr}}_{1:n}}[\tilde{R}_{S_{1:n}}(U)] - \lambda\mathbb{E}_{U,S^{\text{va}}_{1:n}|S^{\text{tr}}_{1:n}}[R_\tau(U)] \leq I^{S^{\text{tr}}_{1:n}}(\Phi;S^{\text{va}}_{1:n}) + \psi_{\tilde{\Phi}, S^{\text{va}}_{1:n}|S^{\text{tr}}_{1:n}}(\lambda)
\]

Since $\ell(\tilde{W}, Z)$ is $\sigma$-subgaussian, and $\tilde{\Phi}\indep{S^{\text{va}}_{1:n}}$, $f(\tilde{\Phi}, S^{\text{va}}_{1:n}) =  \frac{1}{n}\sum_{i=1}^{n} \frac{1}{m_{\text{va}}}\sum_{j=1}^{m_{\text{va}}} \ell(\tilde{W}_i, Z_{ij})$ is $\frac{\sigma}{\sqrt{nm_{\text{va}}}}$-subgaussian. Hence, 
$\psi_{\tilde{\Phi}, S^{\text{va}}_{1:n}|S^{\text{tr}}_{1:n}}(\lambda) \leq \frac{\lambda^2\sigma^2}{2nm_{\text{va}}}, \forall \lambda \in \Reals$.
For $\lambda < 0$ we have
\[\mathbb{E}_{U,S^{\text{va}}_{1:n}|S^{\text{tr}}_{1:n}}[R_\tau(U) - \tilde{R}_{S_{1:n}}(U)] \leq \frac{I^{S^{\text{tr}}_{1:n}}(\Phi;S^{\text{va}}_{1:n}) + \psi_{\tilde{\Phi}, S^{\text{va}}_{1:n}|S^{\text{tr}}_{1:n}}(\lambda)}{-\lambda}\leq\sqrt{\frac{2\sigma^2I^{S^{\text{tr}}_{1:n}}(\Phi;S^{\text{va}}_{1:n})}{nm_{\text{va}}}}\]

Similarly, for $\lambda > 0$ we have  
\[\mathbb{E}_{U,S^{\text{va}}_{1:n}|S^{\text{tr}}_{1:n}}[R_\tau(U) - \tilde{R}_{S_{1:n}}(U)] \geq \frac{I^{S^{\text{tr}}_{1:n}}(\Phi;S^{\text{va}}_{1:n}) + \psi_{\tilde{\Phi}, S^{\text{va}}_{1:n}|S^{\text{tr}}_{1:n}}(\lambda)}{-\lambda}\geq-\sqrt{\frac{2\sigma^2I^{S^{\text{tr}}_{1:n}}(\Phi;S^{\text{va}}_{1:n})}{nm_{\text{va}}}}\]

Then, the following concludes the proof:
\[
|\text{gen}_{\text{meta}}^{\text{alt}}(\tau,\mathcal{A}_{\text{meta}}, \mathcal{A}_{\text{base}})| = \mathbb{E}_{S^{\text{tr}}_{1:n}}|\mathbb{E}_{U,S^{\text{va}}_{1:n}|S^{\text{tr}}_{1:n}}[R_{\tau}(U) - \tilde{R}_{S_{1:n}}(U)]| \leq \mathbb{E}_{S^{\text{tr}}_{1:n}} \sqrt{\frac{2\sigma^2I^{S^{\text{tr}}_{1:n}}(\Phi;S^{\text{va}}_{1:n})}{nm_{\text{va}}}}
\]
\end{proof}

\subsection{Proof of Theorem 6.1}\label{proof_thm_joint_sgld}
\begin{theorem*}
Based on Theorem 5.1, for the SGLD algorithm that satisfies Assumptions 1 \& 2, the mutual information for joint training satisfies

\hspace{4cm}$I(\Phi;S_{1:n}) \leq \sum_{t=1}^T \frac{nd+k}{2} \log(1 + \frac{\eta_t^2L^2}{(nd + k)\sigma_t^2})\, .$

Specifically, if $\sigma_t=\sqrt{\eta_t}$, and $\eta_t = \frac{c}{t}$ for $c>0$, we have:

\hspace{4cm}$|\text{gen}_{\text{meta}}^{\text{joi}}(\tau, \mathcal{A}_{meta}, \mathcal{A}_{base})| \leq \frac{\sigma L}{\sqrt{nm}}\sqrt{c\log T + c}\, .$

\end{theorem*}

\begin{proof}
Define the sequence of parameters for $T$ iterations as $\Phi^{[T]}\eqdef(\Phi^1,...,\Phi^T)$ and the corresponding sequence of samplings as $B_{1:n}^{[T]}\eqdef(B_{1:n}^1,...,B_{1:n}^T)$. The output of the algorithm is defined as $\Phi=f(\Phi^{[T]})$, which can be the last iterate $\Phi^T$ or the average output $\frac{1}{T}\sum_{t=1}^T \Phi^t$. From the figure about the parameter updating strategy for joint training illustrated in Section 6, we get the following Markov chain:
$$S_{1:n}\rightarrow B_{1:n}^{[T]}\rightarrow \Phi^{[T]}\rightarrow \Phi\,.$$  
Therefore, by applying Lemma \ref{Lemma_B.7}.7 to the above Markov chain, we have:
\[
I(\Phi;S_{1:n}) \leq I(\Phi^{[T]};S_{1:n}) \leq I(\Phi^{[T]};B^{[T]}_{1:n})= \sum_{t=1}^T I(\Phi^t; B^{[T]}_{1:n}|\Phi^{[t-1]})\,.
\]
The last equality comes from the mutual information chain rule.
Combing the sample strategy with Assumption 1 and the update rule, we obtain:
\[
\begin{aligned}
I(\Phi^t; B^{[T]}_{1:n}|\Phi^{[t-1]}) &=I(\Phi^t; B^t_{1:n}|\Phi^{t-1})\\
&= h(\Phi^t|\Phi^{t-1}) - h(\Phi^t|\Phi^{t-1},B^t_{1:n})\, .
\end{aligned}
\]

Conditioned on $\Phi^{t-1}=\phi^{t-1}$, we have $\Phi^t = \phi^{t-1} -\eta_t G(\phi^{t-1}, B^{t}_{1:n}) + \xi^t$. Then 
\[h(\Phi^t-\phi^{t-1}|\Phi^{t-1}=\phi^{t-1}) = h(\Phi^t|\Phi^{t-1}=\phi^{t-1})\, .\]
Note that $\xi^t$ and $\eta_t G(\phi^{t-1}, B^{t}_{1:n})$ are independent. So we have
\[
\mathbb{E}(||\Phi^t -\phi^{t-1}||_2^2) = \mathbb{E}(||\eta_t G(\phi^{t-1}, B^{t}_{1:n})||_2^2 + ||\xi^t||_2^2) \leq \eta_t^2L^2 + (nd+k)\sigma_t^2\, .
\]

The Gaussian distribution is the one having the largest entropy among the variables with the same second order moment. Hence, 
\[
h(\Phi^t|\Phi^{t-1}=\phi^{t-1}) \leq \frac{nd+k}{2}\log(2\pi e\frac{\eta_t^2L^2+(nd+k)\sigma_t^2}{(nd+k)})\]
for all $\phi^{t-1}$.

In addition,

\[\begin{aligned}
h(\Phi^t|\Phi^{t-1},B^t_{1:n}) &= h(\Phi^{t-1}-\eta_t G(\Phi^{t-1}, B^{t}_{1:n}) + \xi^t|\Phi^{t-1},B^t_{1:n}) \\
&= h(\xi^t) = \frac{nd+k}{2}\log2\pi e\sigma_t^2.
\end{aligned}\]

So we obtain 
\[I(\Phi;S_{1:n}) \leq \sum_{t=1}^T \frac{nd+k}{2} \log(1 + \frac{\eta_t^2L^2}{(nd+k)\sigma_t^2}) \leq \sum_{t=1}^T \frac{\eta_t^2L^2}{2\sigma_t^2}.\]

Hence, for the SGLD algorithm with $\sigma_t=\sqrt{\eta_t}$, constant $c>0$, $\eta_t = \frac{c}{t}$; since $\sum_{t=1}^T\frac{1}{t}\leq \log T + 1$, we have
\[
\begin{aligned}
|\text{gen}^{\text{joi}}_{\text{meta}}(\tau,\mathcal{A}_{\text{meta}}, \mathcal{A}_{\text{base}})| &\leq \sqrt{\frac{2\sigma^2(I(U,W_{1:n};S_{1:n}))}{nm}}\\
&\leq\sqrt{\frac{\sigma^2}{nm}\sum_{t=1}^T \frac{\eta_t^2L^2}{2\sigma_t^2}}\\
&\leq \frac{\sigma L}{\sqrt{nm}}\sqrt{c\log T + c}\, .
\end{aligned}
\]
\end{proof}

\subsection{Proof of Theorem 6.2}\label{proof_thm_alt_sgld}
\begin{theorem*}
Based on Theorem 5.2, for the Meta-SGLD that satisfies Assumption 1, if we set $\sigma_t=\sqrt{2\eta_t/\gamma_t}$, $\sigma_{t,k}=\sqrt{2\beta_{t,k}/\gamma_{t,k}}$, where $\gamma_t$ and $\gamma_{t,k}$ are the inverse temperatures. The meta generalization error for alternate training satisfies 
\[
|\text{gen}^{\text{alt}}_{\text{meta}}(\tau, \text{SGLD}, \text{SGLD})| \leq 
\sqrt{\frac{2\sigma^2I(U,W_{1:n};S^{\text{va}}_{1:n}|S^{\text{tr}}_{1:n})}{nm_{\text{va}}}}\ 
 \leq\ \frac{\sigma}{\sqrt{nm_{\text{va}}}}\sqrt{ \epsilon_U + \epsilon_W}\, ,
\]
where 
\[\epsilon_U = \sum_{t=1}^T \mathbb{E}_{B^{\text{va}}_{I_t},B^{\text{tr}}_{I_t},W_{I_t},U^{t-1}} \frac{\eta_t\gamma_t\|\epsilon_t^u\|^2_2}{2},~~~~\epsilon_W=\sum_{t=1}^T\sum_{i=1}^{|I_t|} \sum_{k=1}^K \mathbb{E}_{{B^{\text{va}}_{i,t,k},B^{\text{tr}}_{i,t,k}, W_{i,t}^{k-1}}} \frac{\beta_{t,k}\gamma_{t,k}\|\epsilon_{t,i,k}^w\|^2_2}{2}\, .\]
\end{theorem*}

\begin{proof}

To prove the above theorem, we need to introduce some basic notations to present the sampling results and the intermediate output of each gradient step, by which we can apply the Markov structure and the mutual information chain rule.
\begin{itemize}
    \item for $K$ inner iterations:
    \begin{itemize}
        \item The sequence of validation data samplings at outer iteration $t$ for task $i$ and the task batch:
        \[B_{i,t,[K]}^{\text{\text{va}}} = (B_{i,t,1}^{\text{va}},...,B_{i,t,K}^{\text{va}}) ,B_{I_t,[K]}^{\text{va}} = (B_{1,t,[K]}^{\text{va}},..., B_{|I_t|,t,[K]}^{\text{va}})\]
        \item The sequence of train data samplings at outer iteration $t$ for task $i$ and the task batch:
        \[B_{i,t,[K]}^{\text{tr}} = (B_{i,t,1}^{\text{tr}},...,B_{i,t,K}^{\text{tr}}),B_{I_t,[K]}^{\text{tr}} = (B_{1,t,[K]}^{\text{tr}},..., B_{|I_t|,t,[K]}^{\text{tr}})\]
        \item the sequence of task specific parameters at outer iteration $t$ of task $i$ and the task batch: \[W_{i,t}^{[K]}=(W_{i,t}^1,...,W_{i,t}^K), W_{I_t}^{[K]}=(W_{1,t}^{[K]},...,W_{|I_t|,t}^{[K]})\,;\]
        \item The output of base learner at outer iteration $t$ of task $i$ and the task batch: 
        \[W_{i,t}=g(W_{i,t}^{[K]}), W_{I_{t}} = (W_{1,t}, ... , W_{|I_t|,t})\]
    \end{itemize}
    \item for $T$ outer iterations: 
    \begin{itemize}
        \item The sequence of meta parameters as $U^{[T]} = (U^1, ..., U^T)$;
        \item validation data sequences as $B_{I_{[T]}}^{\text{va}} = (B_{I_1}^{\text{va}},...,B_{I_T}^{\text{va}})$;
        \item train data sequences as $B_{I_{[T]}}^{\text{tr}}=(B_{I_1}^{\text{tr}},...,B_{I_T}^{\text{tr}})$;
        \item Output of meta learner is defined as $U=f(U^{[T]})$;
        \item Output sequence of base learner is defines as $W_{I_{[T]}} = (W_{I_{1}},...,W_{I_{T}})$
    \end{itemize}
    
\end{itemize}

Based on the definition above, we have the following Markov chains:
\begin{align}
\label{eq:u_chain}S_{1:n}^{\text{va}}\rightarrow B^{\text{va}}_{I_{[T]}}\rightarrow U^{[T]}\rightarrow U\\
\label{eq:w_t}S_{1:n}^{\text{tr}}\rightarrow B^{\text{tr}}_{I{[T]}} \rightarrow W_{I_{[T]}} \rightarrow W_{1:n}\\
\label{eq:w_it}B^{\text{tr}}_{I_{t}}\rightarrow B^{\text{tr}}_{I_{t},[K]} \rightarrow W_{I_{t}}^{[K]} \rightarrow W_{I_{t}}\\
\label{eq:w_itk}B^{\text{tr}}_{i,t}\rightarrow B^{\text{tr}}_{i,t,[K]} \rightarrow W_{i,t}^{[K]} \rightarrow W_{i,t}
\end{align}
And the graph model:
\begin{align}
\label{eq:graph}S_{1:n}^{\text{tr}}\rightarrow B^{\text{tr}}_{I{[T]}} \rightarrow ( W_{I_{[T]}}, U^{[T]}) \leftarrow S_{1:n}^{\text{va}}    
\end{align}

In fact, the algorithm has a nest-loop structure, we just list the above simple sub-structures for the first step of the proof. 
By combining the above Markov chains and the independence of the sample strategy, we obtain

\begin{align}
\nonumber I(U,W_{1:n};S^{\text{va}}_{1:n}|S^{\text{tr}}_{1:n}) &\leq I(U^{[T]}, W_{1:n};S^{\text{va}}_{1:n}|S^{\text{tr}}_{1:n})
\leq I(U^{[T]}, W_{I_{[T]}};S^{\text{va}}_{1:n}|S^{\text{tr}}_{1:n})\\
\label{eq:mi_1}&\leq I(U^{[T]}, W_{I_{[T]}};S^{\text{va}}_{1:n}|B^{\text{tr}}_{I_{[T]}})
\leq I(U^{[T]}, W_{I_{[T]}};B^{\text{va}}_{I_{[T]}}|B^{\text{tr}}_{I_{[T]}})
\end{align}

Apply Lemma \ref{Lemma_B.7}.7, the first and the last inequality are obtained with Markov chain (\ref{eq:u_chain}). The second inequality is obtained with (\ref{eq:w_t}). The third inequality comes from Lemma \ref{Lemma_B.8}.8 and the graph model(\ref{eq:graph}).

Furthermore, we can apply (\ref{eq:w_it}), (\ref{eq:w_itk}), the information chain rule together with the updating rules, to obtain the following decomposition:

\begin{align}
\nonumber I(&U^{[T]}, W_{I_{[T]}};B^{\text{va}}_{I_{[T]}}|B^{\text{tr}}_{I_{[T]}})
= \sum_{t=1}^T I(U^t, W_{I_t};B^{\text{va}}_{I_t}|B^{\text{tr}}_{I_{t}}, U^{t-1}, W_{I_{t-1}})\\
\nonumber &\quad\quad= \sum_{t=1}^T \left\{I(W_{I_t};B^{\text{va}}_{I_t}|B^{\text{tr}}_{I_{t}}, U^{t-1}) + I(U^t;B^{\text{va}}_{I_t}|B^{\text{tr}}_{I_{t}}, W_{I_{t}}, U^{t-1})\right\}\\
\nonumber &\quad\quad\leq \sum_{t=1}^T \left\{\sum_{i=1}^b\left[ I(W_{i,t}^{[K]};B^{\text{va}}_{i,t}|B^{\text{tr}}_{i,t},U^{t-1})\right] + I(U^t;B^{\text{va}}_{I_t}|B^{\text{tr}}_{I_{t}}, U^{t-1}, W_{I_{t}}) \right\}\\
\nonumber &\quad\quad\leq \sum_{t=1}^T \sum_{i=1}^b \sum_{k=1}^K I(W_{i,t}^k;B^{\text{va}}_{i,t,k}|U^{t-1},B^{\text{tr}}_{i,t,k}, W_{i,t}^{k-1}) + \sum_{t=1}^T I(U^t;B^{\text{va}}_{I_t}|B^{\text{tr}}_{I_{t}}, U^{t-1}, W_{I_t})\\
\nonumber &\quad\quad= \sum_{t=1}^T \sum_{i=1}^b \sum_{k=1}^K \mathbb{E}_{B^{\text{va}}_{i,t,k},B^{\text{tr}}_{i,t,k}, W_{i,t}^{k-1}} \left[D_{\text{KL}}(P_{W_{i,t}^k|B^{\text{tr}}_{i,t,k},B^{\text{va}}_{i,t,k},W_{i,t}^{k-1}}||P_{W_{i,t}^k|B^{\text{tr}}_{i,t,k},W_{i,t}^{k-1}})\right] \\
\label{eq:mi_2}&\quad\quad\quad\quad + \sum_{t=1}^T \mathbb{E}_{B^{\text{va}}_{I_t},B^{\text{tr}}_{I_{t}},U^{t-1}} \left[D_{\text{KL}}(P_{U^t|B^{\text{va}}_{I_t},B^{\text{tr}}_{I_{t}}, U^{t-1}, W_{I_t}} ||P_{U^t|B^{\text{tr}}_{I_{t}}, U^{t-1}, W_{I_t}})\right]\, .
\end{align}

\begin{remark}
Here, the KL divergence is for every single iteration, it's not for the full trajectory. In addition, the randomness brought by sampling and previous updates is implied by the expectation. To empirically evaluate the bound, we can sample the variables presented in the expectation to compute the KL divergence.
\end{remark}

Recall the following updates rules:
\[
\begin{aligned}
&W^{k}_{i,t} = W^{k-1}_{i,t} - \beta_{t,k} \nabla R_{B^{\text{tr}}_{i,t,k}}(W^{k-1}_{i,t}) + \zeta^{t,k} \\
&U^t = U^{t-1} - \eta_t \nabla \tilde{R}_{B^{\text{va}}_{I_t}}(U^{t-1}) + \xi^t\, .
\end{aligned}
\]

For the SGLD algorithm, we use the typical choices of $\sigma_t=\sqrt{2\eta_t/\gamma_t}$, $\zeta_k=\sqrt{2\beta_{t,k}/\gamma_{t,k}}$, where $\gamma_t$ and $\gamma_{t,k}$ are the inverse temperatures. Then, the update rules give
\[
\begin{aligned}
&P_{U^t|B^{\text{tr}}_{I_{t}}, U^{t-1}, W_{I_t}}\sim \mathcal{N}(U^{t-1} - \eta_t \nabla \tilde{R}_{B^{\text{tr}}_{I_t}}(U^{t-1}), \frac{2\eta_t}{\gamma_t})\\
&P_{U^t|B^{\text{va}}_{I_t},B^{\text{tr}}_{I_{t}}, U^{t-1}, W_{I_t}} \sim \mathcal{N}(U^{t-1} - \eta_t \nabla \tilde{R}_{B^{\text{va}}_{I_t}, B^{\text{tr}}_{I_t}}(U^{t-1}), \frac{2\eta_t}{\gamma_t})\\
&P_{W_{i,t}^k|B^{\text{tr}}_{i,t,k},W_{i,t}^{k-1}} \sim \mathcal{N}(W_{i,t}^{k-1} - \beta_{t,k} \nabla R_{B^{\text{tr}}_{i,t,k}}(W^{k-1}_{i,t}), \frac{2\beta_{t,k}}{\gamma_{t,k}})\\
&P_{W_{i,t}^k|B^{\text{tr}}_{i,t,k},B^{\text{va}}_{i,t,k},W_{i,t}^{k-1}} \sim \mathcal{N}(W_{i,t}^{k-1} - \beta_{t,k} \nabla R_{B^{\text{tr}}_{i,t,k}, B^{\text{va}}_{i,t,k}}(W^{k-1}_{i,t}),\frac{ 2\beta_{t,k}}{\gamma_{t,k}})
\end{aligned}
\]

Let $\epsilon_t^u= \nabla\tilde{R}_{B^{\text{va}}_{I_t}, B^{\text{tr}}_{I_t}}(U^{t-1}) -   \nabla\tilde{R}_{B^{\text{tr}}_{I_t}}(U^{t-1})$, then we have
\begin{align}
\label{eq:mi_3}D_{\text{KL}}(P_{U^t|B^{\text{va}}_{I_t},B^{\text{tr}}_{I_t},W_{I_t}}||P_{U^t|B^{\text{tr}}_{I_{t}}, W_{I_t}}) = \frac{\eta_t^2||\epsilon_t^u||^2_2}{2\sigma_t^2} = \frac{\eta_t\gamma_t||\epsilon_t^u||^2_2}{4}    
\end{align}

Similarly, let $\epsilon_{t,i,k}^w = \nabla R_{\tilde{B}^{\text{tr}}_{i,t,k}, \tilde{B}^{\text{va}}_{i,t,k}}(W^{k-1}_{i,t}) - \nabla R_{\tilde{B}^{\text{tr}}_{i,t,k}}(W^{k-1}_{i,t})$, we have
\begin{align}
\label{eq:mi_4} D_{\text{KL}}(P_{W_{i,t}^k|W_{i,t}^{k-1},B^{\text{va}}_{i,t,k},B^{\text{tr}}_{i,t,k}}||P_{W_{i,t}^k|W_{i,t}^{k-1},B^{\text{tr}}_{i,t,k}})=\frac{\beta_{t,k}\gamma_{t,k}||\epsilon_{t,i,k}^w||^2_2}{4}    
\end{align}

Combine Theorem 5.2 and equations(\ref{eq:mi_1}), (\ref{eq:mi_2}), (\ref{eq:mi_3}),(\ref{eq:mi_4}), we have
\[
\begin{aligned}
&|\text{gen}^{\text{alt}}_{\text{meta}}(\tau, \text{SGLD}, \text{SGLD})| \leq \sqrt{\frac{2\sigma^2(I(U,W_{1:n};S^{\text{va}}_{1:n}|S^{\text{tr}}_{1:n}))}{nm_{\text{va}}}}\\
&\leq \frac{\sigma}{\sqrt{nm_{\text{va}}}}\sqrt{\sum_{t=1}^T \mathbb{E}_{B^{\text{va}}_{I_t},B^{\text{tr}}_{I_t},U^{t-1}, W_{I_t}} \frac{\eta_t\gamma_t||\epsilon_t^u||^2_2}{2} + \sum_{t=1}^T\sum_{i=1}^{|I_t|} \sum_{k=1}^K \mathbb{E}_{{B^{\text{va}}_{i,t,k},B^{\text{tr}}_{i,t,k}, W_{i,t}^{k-1}}} \frac{\beta_t^k\gamma_t^k||\epsilon_{t,i,k}^w||^2_2}{2}}
\end{aligned}
\]

which concludes the proof.
\end{proof}

\section{On Subgaussianity}\label{sub}

We list the two subgaussian assumptions of \citet{xu2017information} and \citet{bu2020tightening} respectively as follows: 

\textbf{Assumption (a)} $\forall w \in \Wcal$, $\ell(w,Z)$ is $\sigma$-subgaussian for $Z \sim \mu$.

\textbf{Assumption (b)} $\ell(\tilde{W},Z)$ is $\sigma$-subgaussian under $ P_{\tilde{W},Z}=P_W\times\mu$, where $\tilde{W}$ is an independent copy of $W$ and $\tilde{W} \indep Z$.

\citet{xu2017information} directly use Assumption (a) to conclude Assumption (b) in their proof. Two counter examples have been proposed to challenge this conclusion in Appendix section C of \cite{negrea2019information} and section IV of \cite{bu2020tightening}.
However, we notice that these two counterexamples are based on the case of unbounded loss with no constraint on the parameter $W$ output by the learning algorithm. 
We now compare the two assumptions mentioned above in detail for unbounded loss and bounded loss.

\subsection{unbounded loss}

\textbf{Counterexample for Assumption (a) => (b)} (\citet{negrea2019information})

Consider $\Wcal=\Zcal=\Reals$ with $\ell(w, z) = w + z$. Assume that $\tilde{W}\indep Z, \tilde{W} \sim Cauchy, \tilde{Z}\sim \Ncal(0,\sigma^2)$. Thus, $\ell(w,Z)$ is $\sigma$-subgaussian for any $w \in \Wcal$, because $\psi_{\ell(w, Z)}(\lambda) = \log\EE_{Z}[e^{\lambda(\ell(w, Z)-\EE\ell(w, Z))}] = \log\EE[e^{\lambda Z}] = \exp^{\frac{\lambda^2\sigma^2}{2}}$.
While $\ell(\tilde{W}, \tilde{Z})$ is not subgaussian since the Cauchy distribution does not have well-defined moments higher than the zeroth moment. 

\textbf{Counterexample for Assumption (b) => (a)} (\citet{bu2020tightening})

Consider $\Wcal = \Zcal = \Reals^d$ and the square loss function $\ell(w, z) = \|w - z\|_2^2$. Assume $\tilde{W}\sim \Ncal(\mu, \sigma_W^2 \mathbb{I}_d), Z\sim \Ncal(\mu,\sigma^2_Z \mathbb{I}_d)$. Then $\ell(w, Z)$ is not subgaussian for all $w\in\Wcal$, since when $\|w\|_2^2 \rightarrow \infty$ the variance of $\ell(w,Z)$ is not bounded. However, $\tilde{W} - Z \sim \Ncal(0, (\sigma_W^2 + \sigma^2_Z)\mathbb{I}_d)$, so $\ell(\tilde{W}, \tilde{Z}) = \|\tilde{W} - Z\|_2^2 \sim (\sigma_W^2 + \sigma^2_Z)\chi^2_d$ has bounded CGF for $\lambda < 0$, which can induce one-sided bound in our theorem with $\sigma^2=2d(\sigma_W^2+\sigma_Z^2)^2$. However, in this condition, the loss is sub-exponential but not subgaussian as claimed by \citet{bu2020tightening} in assumption (b). 

\subsection{bounded loss}
For a bounded loss function $\ell(w,z) \in [a,b]$, the two assumptions are equivalent. $\ell(w, Z)$ is $\frac{(b-a)}{2}$-subgaussian $\forall w \in \Wcal, Z \sim \mu$. Similarly, $\ell(W,Z)$ is $\frac{(b-a)}{2}$-subgaussian under $P_{\tilde{W},Z}=P_W\times\mu$. The counter example of \cite{negrea2019information} does not apply because the Cauchy distribution is truncated and, consequently, has well-defined moments.

\subsection{Discussion}
Based on the above analysis, we can conclude the following. 
For a bounded loss, the two assumptions are equivalent. In contrast, Assumption (b) is a stronger assumption than Assumption (a) when the loss function is unbounded. At the same time, we found that Assumption (b) is also hard to ensure and is often replaced by the sub-exponential assumption as a relaxation for unbounded loss.

What we need for proving Theorem~\ref{thm:mi-jt} and \ref{thm:cmi-al} is actually the extension of assumption (b). However, in practice, the parameters output from an algorithm should always be bounded. Moreover, for complex data sets used in deep learning, people often adopt a bounded loss or truncate the unbounded loss to ensure the theoretical guarantee. The inconsistency between the two assumptions should not cause too many problems. Hence, we extended Assumption (a) to avoid confusion and too much discussion in the main paper, although the more rigorous version should make use of Assumption (b).

\section{Additional Experimental Results}

\subsection{Synthetic Data}\label{toy_add}
In this section, we present a more direct visualization for the 2D mean estimation experiment described in Section 7.1.
We compare the results of three different train-validation split settings in Figure \ref{fig:three graphs}. The yellow cross in the figure is the actual environment mean $(-4,-4)$. Note that we have set the task batch size as $|I_t|=5$. The five clusters in the graph are the task batch data points at the last epoch, which corresponds to five different $\mu_i\sim\tau, \forall i\in[|I_t|]$. We use small dots to represent the data points, and big dots to show the estimated cluster mean $W_i$ and the estimated environment mean $U$. 

Figure \ref{fig:three graphs} illustrates that the distances from the estimated $U$ to the yellow cross are slightly different for these three settings. When $m_{\text{va}} = 1$ the estimated mean $U$ is much closer to the actual environment mean. This result is coherent with the bound estimation results in Section 7.1, where we got the tightest gradient incoherence bound with $m_{\text{va}}=1$. While the gradient norm bound is largest for $m_{\text{va}}=1$, which indicates that the gradient norm bound may not be as reliable as the gradient incoherence bound since it may be much looser and won't give too much information.

\begin{figure}[ht]
     \centering
     \begin{subfigure}[b]{0.3\textwidth}
         \centering
         \includegraphics[width=\textwidth]{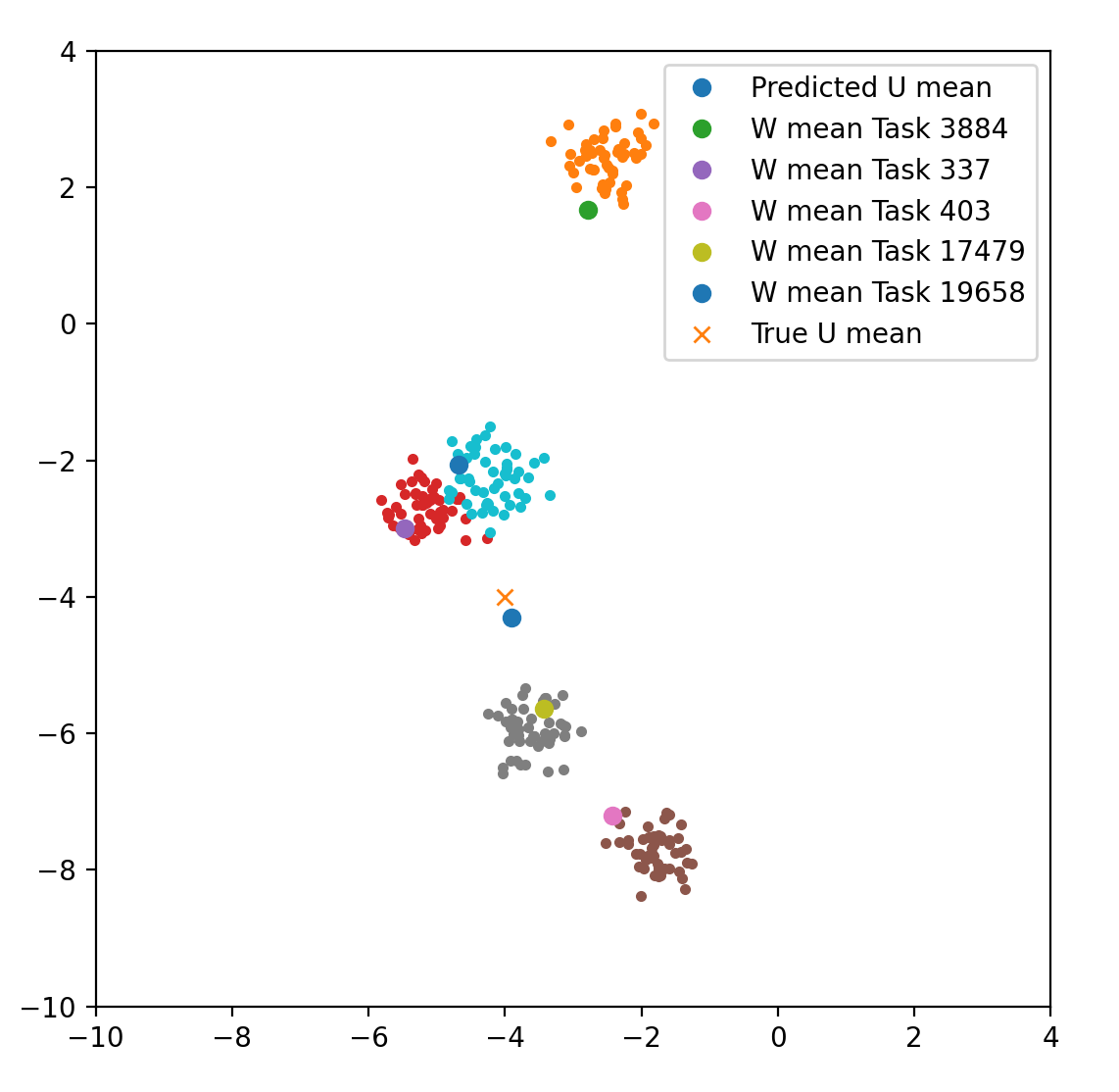}
         \caption{$m_{\text{va}}=15$}
     \end{subfigure}
     \hfill
     \begin{subfigure}[b]{0.295\textwidth}
         \centering
         \includegraphics[width=\textwidth]{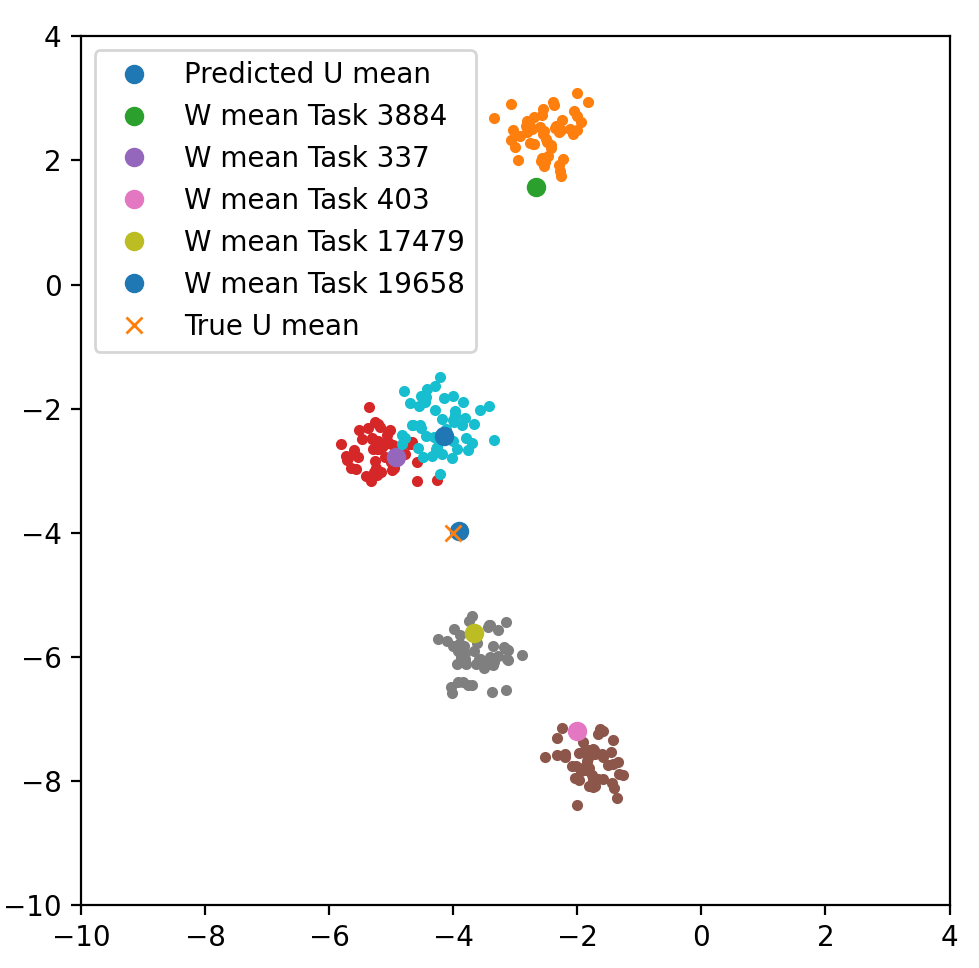}
         \caption{$m_{\text{va}}=1$}
         
     \end{subfigure}
     \hfill
     \begin{subfigure}[b]{0.295\textwidth}
         \centering
         \includegraphics[width=\textwidth]{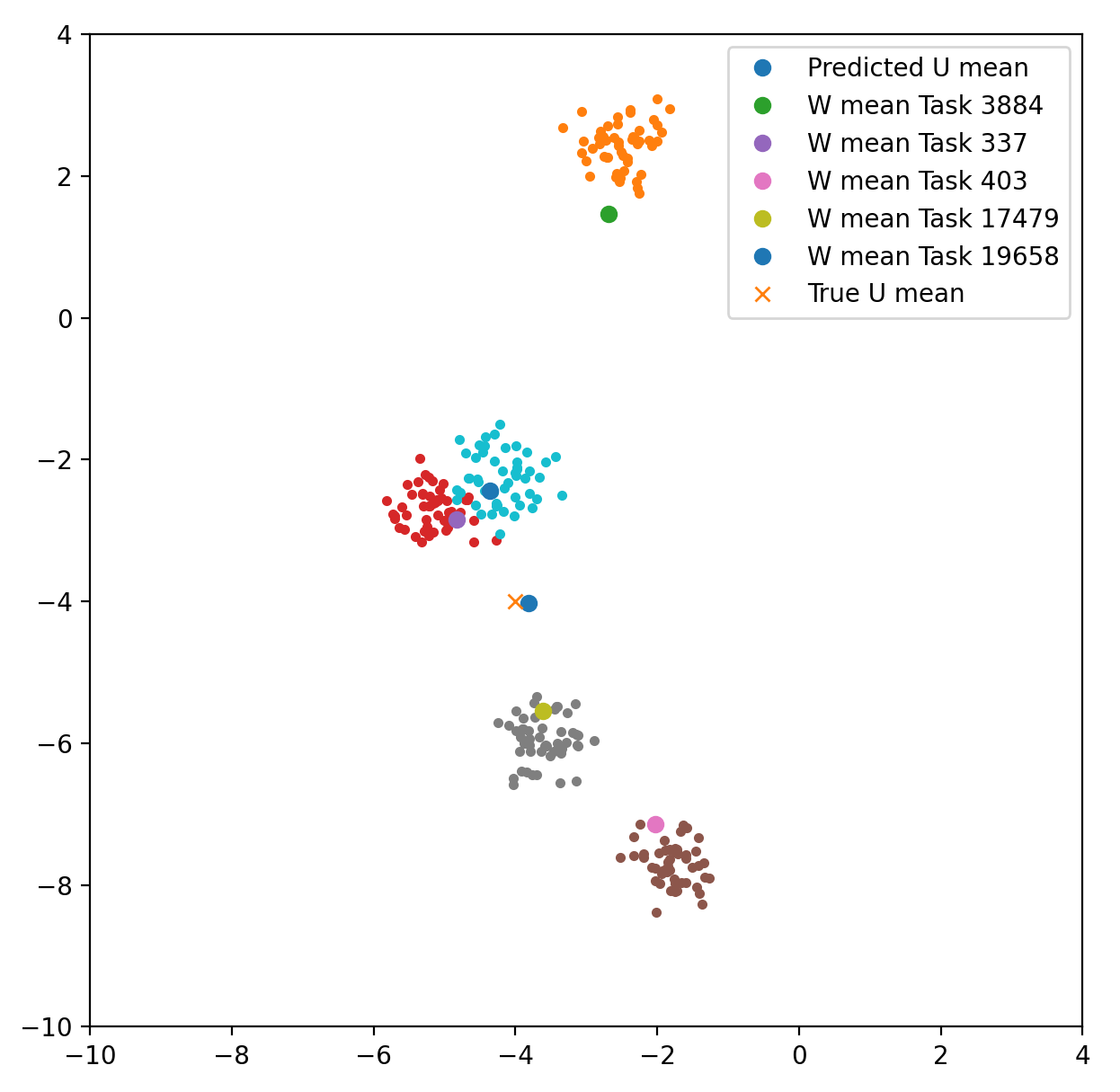}
         \caption{$m_{\text{va}}=8$}
     \end{subfigure}
        \caption{Visualization for simulated data results}
        \label{fig:three graphs}
\end{figure}\label{toy_img}

\textbf{Comparison with the observed generalization error}

We calculated the observed generalization error by evaluating the expected difference between the train loss and test loss. And we list the results of synthetic data under different train-validation split settings in Table \ref{tab:toy_8_8}, Table \ref{tab:toy_15_1} and Table \ref{tab:toy_1_15}.

\begin{table}[!htbp]
\caption{$m_{tr}=8, m_{va}=8$}
\centering
\resizebox{\columnwidth}{!}{%
\rowcolors{2}{white}{lightgray} 
\begin{tabular}{>{\itshape}0l0l0l0l0l0l0l0l0l0l}\hline 
\textup{epoch}       & 20 & 40 & 60 & 80 &100 &120 &140 &160 &180\\\hline
Train-Test gap & 0.0697 & 0.0371 & 0.0908 & 0.0241 & 0.1072 & 0.1492 & 0.1775 & 0.1432 & 0.1581 \\
Lipschitz &19.77&27.63&33.7&38.84&43.37&47.47&51.24&54.75&58.05\\
G\_norm & 6.009 & 7.506 & 8.807 & 9.856 & 10.643 & 11.558 & 12.386 & 13.275 & 14.014 \\
G\_inco (Ours) & 0.251 & 0.3462 & 0.4315 & 0.4976 & 0.5529 & 0.6112 & 0.6480 & 0.6983 & 0.7424 \\\hline
\end{tabular}\label{tab:toy_8_8}
}
\end{table}

\begin{table}[!htbp]
\caption{$m_{tr}=15, m_{va}=1$}
\centering
\resizebox{\columnwidth}{!}{%
\rowcolors{2}{white}{lightgray} 
\begin{tabular}{>{\itshape}0l0l0l0l0l0l0l0l0l0l}\hline 
\textup{epoch}       & 20 & 40 & 60 & 80 &100 &120 &140 &160 &180\\\hline
Train-Test gap &0.06857&0.03689&0.09728&0.0327&0.1113&0.1628&0.04986&0.06426&0.1674 \\
Lipschitz &55.79&77.96&95.09&109.6&122.4&133.9&144.6&154.5&163.8\\
G\_norm & 17.29&21.39&25.07&28.04&30.24&32.79&35.15&37.63&39.73 \\
G\_inco (Ours) & 0.1774&0.2486&0.3127&0.3661&0.4074&0.4481&0.4817&0.5182&0.5468 \\\hline
\end{tabular}\label{tab:toy_15_1}
}
\end{table}

\begin{table}[!htbp]
\caption{$m_{tr}=1, m_{va}=15$}
\centering
\resizebox{\columnwidth}{!}{%
\rowcolors{2}{white}{lightgray} 
\begin{tabular}{>{\itshape}0l0l0l0l0l0l0l0l0l0l}\hline 
\textup{epoch}       & 20 & 40 & 60 & 80 &100 &120 &140 &160 &180\\\hline
Train-Test gap &0.2389&0.06293&0.3155&0.1718&0.1822&0.2196&0.1687&0.1814&0.1923 \\
Lipschitz &14.21&19.86&24.22&27.91&31.17&34.12&36.83&39.35&41.72\\
G\_norm & 4.484&5.579&6.572&7.363&7.953&8.632&9.249&9.899&10.42 \\
G\_inco (Ours)& 0.7801&1.048&1.286&1.45&1.617&1.76&1.89&2.018&2.149 \\\hline
\end{tabular}\label{tab:toy_1_15}
}
\end{table}

Where Train-Test gap is the observed generalization error, G\_inco is the whole gradient incoherence bound, i.e: $\sqrt{\frac{\sigma^2(\epsilon_U + \epsilon_W)}{nm_{va}}}$, G\_norm is the corresponding bound w.r.t. gradient norm. 

Thus, we can see that the gradient-incoherence bound is much closer to the estimation of the actual gap but can be improved in the future.

\subsection{Omniglot}\label{omniglot_add}

Now we give additional experimental results for the deep few-shot benchmark -- Omniglot. We compare the test accuracy for Meta-SGLD with three train-validation split settings, \ie, $m_{va}=\{1,8,15\}$. The test accuracy for MAML and Meta-SGLD with $\{0, 1, 4, 10\}$ fine-tune steps are illustrated in Table \ref{tab:acc}. 

Under the same experiment settings, Meta-SGLD achieves slightly better performance than our reproduced MAML. However, our test accuracy is not comparable to the original results of MAML~\cite{finn2017model}. We only trained the model with 2000 epochs, and the other hyper-parameter settings are also different from \cite{finn2017model}. Moreover, our Meta-SGLD code is modified based on \cite{MAML_Pytorch}. This realization version of MAML is claimed by the author to have worse performance than original MAML. We would like to re-emphasize that our experiments were conducted to validate our theories but not to achieve SOTA results. 

Comparing experimental results for different train-validation split settings, we note that the train loss at last epoch for $m_{va}=1$ is smaller than $m_{va}=8$, while the best test accuracy is obtained with $m_{va} = 8$. Non-rigorously we think the generalization error of $m_{va}=8$ should be smaller than $m_{va}=1$. The consistent result was verified by the gradient-incoherence bound, which is the tightest for $m_{va}=8$. For $m_{va}=15$,  \ie, training with 1-shot data, both the test accuracy, train loss and the estimated bound were the worst.

\begin{table}[ht]
    \centering
    \caption{Test Accuracy for Omniglot, train with 2000 epochs}
    \begin{tabular}{ |p{3cm}||p{3cm}|p{3cm}|p{3cm}|  }
 \hline
 \multicolumn{4}{|c|}{$5$-way Test Accuracy} \\
 \hline
 Algorithm & $m_{va}=15$ &$m_{va}=8$&$m_{va}=1$\\
 \hline
 MAML $0$-step   & 20.7\%    & 20.13\%&   20.26\%\\
 MAML $1$-step   & 88.43\%    &95.8\%&   92.43\%\\
 MAML $4$-step   & 90.77\%    &96.97\%&   96.14\%\\
 MAML $10$-step   & 91.06\%    &97.07\%&   96.53\%\\
 Meta-SGLD $0$-step &   19.48\%  & 20.06\%   &20.29\%\\
 Meta-SGLD $1$-step &   88.8\%  & 95.95\%   &92.8\%\\
 Meta-SGLD $4$-step &  91.1\%  & 96.97\%   &96.1\%\\
 Meta-SGLD $10$-step &   91.26\%  & 97.1\%   &96.53\%\\
 \hline
\end{tabular}
    \label{tab:acc}
\end{table}

\textbf{Comparison with the observed generalization error}

Similar to the synthetic setting, we calculated the observed generalization error by evaluating the expected difference between the train loss and test loss. And we list the results of Omniglot data under different train-validation split setting in the following Table~\ref{tab:omg_8_8}, \ref{tab:omg_15_1} and \ref{tab:omg_1_15}:

\begin{table}[!ht]
\caption{$m_{tr}=8, m_{va}=8$}
\centering
\resizebox{\columnwidth}{!}{%
\rowcolors{2}{white}{lightgray} 
\begin{tabular}{>{\itshape}0l0l0l0l0l0l0l0l0l0l}\hline 
\textup{epoch} & 200 & 400 & 600 & 800 & 1000 & 1200 & 1400 & 1600 & 1800\\\hline
Train-Test gap & 0.01896&0.00364&0.008821&0.01856&0.01366&0.0001578&0.04087&0.02269&0.01669\\
Lipschitz &4.8159&6.8108&8.3415&9.6319&10.7688&11.7966&12.7418&13.6215&14.4478\\
G\_norm & 0.1835&0.2765&0.3578&0.4292&0.4959&0.5557&0.6134&0.6679&0.7204\\
G\_inco (Ours) & 0.109&0.1372&0.1617&0.1841&0.2057&0.2252&0.2444&0.2625&0.2798 \\\hline
\end{tabular}\label{tab:omg_8_8}
}
\end{table}

\begin{table}[!ht]
\caption{$m_{tr}=15, m_{va}=1$}
\centering
\resizebox{\columnwidth}{!}{%
\rowcolors{2}{white}{lightgray} 
\begin{tabular}{>{\itshape}0l0l0l0l0l0l0l0l0l0l}\hline 
\textup{epoch} & 200 & 400 & 600 & 800 & 1000 & 1200 & 1400 & 1600 & 1800\\\hline
Train-Test gap & 0.01508&0.02493&0.1099&0.07129&0.01425&0.07677&0.007417&0.04808&0.04199 \\
Lipschitz &9.3429&13.2129&16.1824&18.6858&20.8914&22.8854&24.719&26.4258&28.0287\\
G\_norm & 0.4536&0.6265&0.8012&0.9641&1.115&1.265&1.409&1.55&1.688\\
G\_inco (Ours) & 0.123&0.1572&0.2146&0.2787&0.3437&0.4154&0.4775&0.5275&0.5848 \\\hline
\end{tabular}\label{tab:omg_15_1}
}
\end{table}

\begin{table}[!ht]
\caption{$m_{tr}=1, m_{va}=15$}
\centering
\resizebox{\columnwidth}{!}{%
\rowcolors{2}{white}{lightgray} 
\begin{tabular}{>{\itshape}0l0l0l0l0l0l0l0l0l0l}\hline 
\textup{epoch} & 200 & 400 & 600 & 800 & 1000 & 1200 & 1400 & 1600 & 1800\\\hline
Train-Test gap &0.0474&0.001597&0.02478&0.01946&0.02008&0.006832&0.05856&0.09882&0.0331 \\
Lipschitz &45.8866&64.8935&79.4779&91.7732&102.6056&112.3988&121.4046&129.7869&137.6598\\
G\_norm & 0.9537&0.9639&0.9756&0.9861&0.9974&1.011&1.025&1.039&1.053\\
G\_inco (Ours) & 0.9534&0.9619&0.971&0.9785&0.9863&0.9959&1.006&1.015&1.025 \\\hline
\end{tabular}\label{tab:omg_1_15}
}
\end{table}

Where Train-Test gap is the observed generalization error,                       G\_inco is the whole gradient incoherence bound, i.e: $\sqrt{\frac{\sigma^2(\epsilon_U + \epsilon_W)}{nm_{va}}}$, G\_norm is the corresponding bound w.r.t. gradient norm.

\section{Experiment Details}
Although we have described the detailed algorithm in the main paper to obtain a data-dependent estimate bound, we offer a more structural pseudo-code in section \ref{alg:meta-sgld}. 
We used Monte Carlo simulations to estimate our generalization error bound in Theorem 6.2. Recall the  accumulated gradient incoherence  for meta learner and base learner are respectively denoted as:
\[\epsilon_U = \sum_{t=1}^T \mathbb{E}_{B^{\text{va}}_{I_t},B^{\text{tr}}_{I_t},W_{I_t},U^{t-1}} \frac{\eta_t\gamma_t\|\epsilon_t^u\|^2_2}{2},~~~~\epsilon_W=\sum_{t=1}^T\sum_{i=1}^{|I_t|} \sum_{k=1}^K \mathbb{E}_{{B^{\text{va}}_{i,t,k},B^{\text{tr}}_{i,t,k}, W_{i,t}^{k-1}}} \frac{\beta_{t,k}\gamma_{t,k}\|\epsilon_{t,i,k}^w\|^2_2}{2}\, .\]
In our experiments, the two terms are separately estimated.
Since we have \[\tilde{R}_{B^{\text{va}}_{I_t}}(U^{t-1}) = \frac{1}{|I_t|}\sum_{i\in I_t} R_{B^{\text{va}}_{i,t}}(W^{K}_{i,t})\,,\]
$\epsilon_t^u= \nabla\tilde{R}_{B^{\text{va}}_{I_t}, B^{\text{tr}}_{I_t}}(U^{t-1}) -   \nabla\tilde{R}_{B^{\text{tr}}_{I_t}}(U^{t-1})$ is related to the last inner step output $W^{K}_{i,t}$. To estimate $\epsilon_U$, we conducted 10 times Monte Carlo simulations for the corresponding inner path at each iteration $t$, the gradients are calculated with back-propagation. For $\epsilon_W$, it's much simpler, we just conducted 10 times Monte Carlo simulations at each inner step, see more details in the code. Our code is modified based on \citet{MAML_Pytorch} and \citet{MALP}. 

\subsection{Synthetic Data}\label{toy_exp}

\textbf{Network Structure} For Synthetic Data, the model structure is quite simple. It is a 2D mean estimation. For a single task with parameter $w$, to estimate the mean, we need to calculate the loss $\ell(W,Z)=||W-Z||_2^2$. Hence, we constructed a single layer that conducts $W - Z$. Then the output of this layer and the pseudo target (always set to 0) were taken as input to a square loss function. 

\begin{table}[hbt!]
\centering
\begin{tabular}{c|c}
\hline
Hyper parameters & values\\
task numbers $n$ & 20000\\
sample numbers $m$ & 16 \\
outer Loop inverse temperature $\gamma_t$ & 10000\\
Inner Loop inverse temperature $\gamma_{t,k}$ & 10000\\
Outer Loop learning rate $\eta_t$ & $0.2$\\
Inner Loop learning rate $\beta_{t,k}$ & $0.4$\\
task batch size $|I_t|$ & 5\\
epoch/Outer Loop iterations $T$ & 200\\
Inner Loop updates $K$ & 4\\
$m_{\text{va}}$ & $\{1, 8, 15\}$\\
$m_{\text{tr}}$ & $\{15, 8, 1\}$\\
data dimension & 2\\
loss & square loss\\
test update step & 10\\
\hline
\end{tabular}
\caption{\label{tab:synthetic} Synthetic Data Experiment Setting}
\end{table}

\textbf{Training Details} The hyper parameter settings and training details for Synthetic data set are presented in Table \ref{tab:synthetic}. 

\textbf{Compute Resource} All experiments for Synthetic data were tested on a machine runing macOS system with an Intel Core i5 CPU, 8G memory.

\textbf{Subgaussian parameter}
For the synthetic data, we want to estimate the mean for each sub-task, where we have for task $i$, $Z\sim \Ncal(\mu_i, 0.1\mathbb{I}_d), d=2$. The task mean $\mu_i$ is sampled from the truncated normal distribution $\Ncal((-4,-4)^T, 5\mathbb{I}_{2})$ with $\mu_i\in [-12, 4]\times[-12, 4]$. Thus we have $||\mu_i||_2^2 \leq 288$.
$\tilde{W}$ is the independent copy of the SGLD algorithm output $W$. To estimated the $\sigma^2$ that satisfies the subgaussian loss, We consider the worst case where the output is obtained with a single example $Z'$ and one inner step update. So we have $W = \textbf{0} -2\beta(\textbf{0} - Z') + \epsilon \approx 0.8Z'$, since the inner loop learning rate in our experiment setting is $0.4$(the noise added is quite small, which can be ignored).  Hence, $\tilde{W} \sim \Ncal(0.8\mu_i, 0.064\mathbb{I}_d)$. Moreover, we have $\tilde{W} \indep Z$ and $\ell(\tilde{W}, Z)= ||\tilde{W} - Z||_2^2$, so $\tilde{W} - Z \sim \Ncal(0.2\mu_i, \sigma_l^2\mathbb{I}_d), \sigma_l^2 = 0.164$. Furthermore, $\ell(\tilde{W}, Z) \sim \sigma_l^2\prime{\chi}_d^2(k), k=0.04||\mu_i||_2^2$, which is a noncentral chi-squared distribution(\citet{bu2020tightening} analyzed ERM, where $\ell(\tilde{W}, Z)$ follows central chi-squared distribution). Thus the CGF of $\ell(\tilde{W}, Z)$ is given by:
\[
\begin{aligned}
\psi_{\ell(\tilde{W}, Z)}(\lambda) &= -(d+k)\sigma_l^2\lambda -\frac{d}{2}\log(1-2\sigma_l^2\lambda) + \frac{k\sigma_l^2\lambda}{1-2\sigma_l^2\lambda}\\
&= \frac{d}{2}(-2\sigma_l^2\lambda -\log(1-2\sigma_l^2\lambda)) + k\sigma_l^2\lambda\frac{2\sigma_l^2\lambda}{1-2\sigma_l^2\lambda}, \lambda \in (-\infty, \frac{1}{2\sigma_l^2})
\end{aligned}\,
\]

Let $u\eqdef2\sigma_l^2\lambda$, and note that $-u-\log(1-u)\leq \frac{u^2}{2}, u<0$.
\[\psi_{\ell(\tilde{W}, Z)}(\lambda) = \frac{d}{2}(-u-\log(1-u)) + \frac{ku^2}{2(1-u)}\leq \frac{du^2}{4} + \frac{ku^2}{2}= (2k+d)\sigma_l^4\lambda^2, \lambda < 0\,.\] So the subgaussian parameter $\sigma^2$ in our assumption can be expressed as $\sigma^2=2(2k+d)\sigma_l^4= 2(2*0.04||\mu_i||_2^2+d)(0.164)^2$, where $d=2$ and $||\mu_i||\leq 288$. So we obtain $\sigma^2 = 0.164*0.164*4*(1+0.04*288)= 1.3469$.

\subsection{Omniglot}\label{omniglot}

\textbf{Network Structure} We used a CNN network architecture for Omniglot data set, which consists of a stack of modules. The first three modules are the same, each of which is a $3\times3$ 2d convolution layer of 64 filters and stride 2 followed by a Relu layer and a batch normalization layer. Then the fourth module is a $2\times2$ 2d convolution layer of 64 filters and stride 1, followed by a Relu layer and a batch normalization layer. Through the aforementioned modules, we got a $64\times 1\times 1$ feature map. This feature map was further taken into a fully connected layer which output the logits for a $5$-way classification. Finally,  the cross-entropy loss is calculated with the logits and the corresponding labels.  

\textbf{Training Details} The hyper parameter settings and training details for Omniglot data set are outlined in Table \ref{tab:Omniglot}. 

\textbf{Compute Resource} The experiments for Omniglot were run on a server node with 6 CPUs and 1 GPU of 32GB memory.

\textbf{Subgaussian parameter} For Omniglot data, we used the cross entropy loss, which is unbounded. And the data distribution is too complex that we cannot obtain a similar closed form estimation for the subgaussian parameter. To assure the theoretic guarantee, we can adopt a variation of the loss function which is clipped to $[0,2]$ and hence $1$-subgaussian. Actually, such clip is not always necessary. As we discussed in section \ref{sub}, the subgaussian parameter $\sigma^2$ is related to the independent copy $\tilde{W}$ of the base learner output for each task. During our experiments, the loss w.r.t $\tilde{W}$ rarely exceed the clip value. 

\begin{table}[!hbt]
\centering
\begin{tabular}{c|c}
\hline
Hyper parameters & values\\
task numbers $n$ & $\tbinom{1200}{5}$\\
sample numbers $m$ & 16 \\
outer Loop inverse temperature $\gamma_t$ & 100000000\\
Inner Loop inverse temperature $\gamma_{t,k}$ & 100000000\\
Outer Loop learning rate $\eta_t$ & $10^{-3}*0.96^{\frac{t}{800}}$\\
Inner Loop learning rate $\beta_{t,k}$ & $0.3*0.96^{\frac{t}{1000}}$\\
n-way classification & 5\\
task batch size $|I_t|$ & 32\\
epoch/Outer Loop iterations $T$ & 2000\\
Inner Loop updates $K$ & 4\\
$m_{\text{va}}$ & $\{1, 8, 15\}$\\
$m_{\text{tr}}$ & $\{15, 8, 1\}$\\
loss & cross entropy\\
test update step & 10\\
image size & 28*28\\
image channel & 1\\
\hline
\end{tabular}
\caption{\label{tab:Omniglot} Omniglot Experiment Setting}
\end{table}

\section{Additional Comparison to Related Works}

\textbf{Discussion with \citet{jose2020information}} \label{jose_related}
They adopted different and generally unrealistic assumptions to derive the theoretical results. Concretely:

In joint-training (Eq (33) in \citet{jose2020information}), the task-level error w.r.t. base-learner $W$ is related to the unknown environment distribution $P_T$, which is hard to estimate from the observed data. In contrast, the task-level risk in our paper is associated with the distribution meta-parameter $U$, which can be evaluated efficiently. Besides, when $m\to\infty$ and the number of task $n$ is limited, their bound always has a non-zero term. This does not fit the reality since the new task already has enough samples to learn.

In the alternate-training (meta train-validation) settings, they assumed the task parameters $W$ and $S^{va}$ are conditionally independent given $S^{tr}$ (Eq A(8) in their paper). This is an unrealistic condition in meta-learning since $W$ depends on the meta-parameter $U$, where $U$ is updated by $S^{va}_{1:n}$. As a result, if we set $m=1$ (each task has only one sample), then $n\to\infty$, the upper bound in Eq(3) of \cite{jose2020information} will converge to 0, which is problematic since task distribution can be arbitrary noisy and the task-level error (with one sample) can be quite large. Besides, this bound is irrelevant to the train validation split, which is inconsistent with the previous work such as \cite{denevi2018learning,saunshi2021representation}.

Therefore, our theoretical results are not directly comparable. Even if we ignore all these unrealistic theoretical assumptions and directly compare the results in \citet{jose2020information}, their theoretical results in noisy iterative approaches still depend on the Lipschitz constant of the neural network (Eq~(45) in their paper), which is vacuous in deep learning.

\textbf{Discussion with recent theoretical analysis on the support-query approach}\label{sq-related}

\citet{denevi2018learning} first studied train-validation split for meta-learning in biased linear regression model. They proved a generalization bound and concluded that there exists a trade-off for train-validation split, which is consistent with Theorem 5.2 in our paper. Specifically, they constructed two datasets: For the simple unimodal distribution, the optimal split is $m_{tr}=0$. For the bimodal distribution, the optimal split is $m_{tr}\in(0, m-1]$.

\citet{bai2021important} proposed a theoretical analysis of train-validation split in linear centroid meta-learning (parameter transfer). By comparing the train-val (alternate training) and train-train (joint training) method, they showed that train-validation split is necessary for the agnostic setting, where the train-val meta loss is an unbiased estimator w.r.t. the meta-test loss while the train-train loss is biased(consistent with our Theorem 5.1). When it is realizable (noiseless scenario), the train-train model can achieve better excess loss.

\citet{saunshi2021representation} analyze the train-valid splitting for linear representation learning (representation transfer). They proved that the train-validation split encourages learning a low-rank representation. In the noiseless setting, the train-val method already enables low-rank representation, so it's preferable to set a smaller train-split and larger validation-split.

While our work focus on general settings with randomized algorithms and does not specify the form of base-learner and meta-learner, which can be applied in non-linear representation, non-linear classifier, and non-convex loss. Besides, the relations of our papers are as follows: 

1. Our theory can recover the stochastic version of the above parameter and representation transfer settings. If we consider the linear model with $\mathcal{U}=\mathcal{W} \subseteq R^d$ and $P_{W|U}$ is approximated by a Gaussian distribution $\mathcal{N}(U,\mathbb{I}_{d})$, the problem is analogous to parameter-transfer meta-learning. If $\mathcal{U}\subseteq R^k, \mathcal{W} \subseteq R^{k+d}$ (where $U \in R^k$ is the shared representation parameter, $V \in R^d$ is the parameter of the linear classifier, $W=(U,V) \in R^{(k+d)}$ is the whole task parameter) and the prior of stochastic linear classifier $V $ is approximated by a Gaussian distribution $\mathcal{N}(0,\mathbb{I}_{d})$, the setting is similar to the representation transfer paradigm.
   
2. Since our bounds are based on the generic settings (flexible data distribution, algorithm, and loss choice), the two training modes are not directly comparable in our problem. However, we agree on the potential limit of joint training (asymptotically biased in the agnostic setting) and believe it is highly interesting to explore the specific conditions to understand the benefits and limitations of these training modes as the future work.

\section{Pseudo Code}\label{alg:meta-sgld}

\begin{algorithm}[H]
\SetAlgoLined
\textbf{Require}: Task environment $\tau$\;
\textbf{Require}: initial learning rates $\eta_0, \beta_0$, inverse temperature $\gamma$\;
randomly initialize $U^0$\;
\For{$t \gets 1$ to $T$}{
  Sample task data batch $B_i \sim \mu_{m,\tau}, \forall i \in I_t$\;
  Randomly split $B_{I_t}$ to $B_{I_t}^{tr}$ and $B_{I_t}^{va}$\;
  learning rate decay, get $\eta_t$\;
  \For{$i \gets 1$ to $|I_t|$}{
     \For{$k \gets 1$ to $K$}{
        learning rate decay, get $\beta_{t,k}$\;
        \eIf{GLD}{ 
           Use full batch, $B_{i,t,k}^{tr} = B_{i,t}^{tr}$ and $B_{i,t,k}^{tr} = B_{i,t}^{va}$\;
        }{
           Sample $B_{i,t,k}^{tr}$ from $B_{i,t}^{tr}$\;
           Sample $B_{i,t,k}^{va}$ from $B_{i,t}^{va}$\;
        }
        Update parameter with gradient descent:\;
        \If{$k==1$}{
            $W_{i,t}^{k-1}=U^{t-1}$\;
        }
        Calculate $\mathbb{E}_{{B^{\text{va}}_{i,t,k},B^{\text{tr}}_{i,t,k}, W_{i,t}^{k-1}}} \frac{\beta_t^k\gamma_t^k||\epsilon_{t,i,k}^w||^2_2}{2}$ with Monte Carlo simulation\;
        $W^{k}_{i,t} = W^{k-1}_{i,t} - \beta_{t,k} \nabla R_{B^{\text{tr}}_{i,t,k}}(W^{k-1}_{i,t}) + \zeta^{t,k}$\;
     }
  }
  Calculate $\mathbb{E}_{B^{\text{va}}_{I_t},B^{\text{tr}}_{I_t},U^{t-1}, W_{I_t}} \frac{\eta_t\gamma_t||\epsilon_t^u||^2_2}{2}$ with Monte Carlo Simulation\;
  $U^t = U^{t-1} - \eta_t \nabla \frac{1}{|I_t|}\sum_{i\in I_t} R_{B^{\text{va}}_{i,t}}(W^{K}_{i,t}) + \xi^t$
}
\caption{Meta-SGLD for Few-Shot Learning}
\end{algorithm}

\end{document}